\documentclass[onefignum,onetabnum]{siamart190516}

\usepackage[utf8]{inputenc} 
\usepackage[T1]{fontenc}    
\usepackage{hyperref}       
\usepackage{url}            
\usepackage{booktabs}       
\usepackage{amsfonts}       
\usepackage{nicefrac}       
\usepackage{microtype}      
\usepackage{xcolor}         

\usepackage{lipsum}
\usepackage{graphicx}
\usepackage{epstopdf}
\usepackage{mathtools}        
\usepackage{makecell}        
\usepackage[noend]{algorithmic}
\usepackage{bm}
\usepackage{lmodern}
\newsiamremark{remark}{Remark}
\newsiamremark{hypothesis}{Hypothesis}
\crefname{hypothesis}{Hypothesis}{Hypotheses}
\newsiamthm{claim}{Claim}
\newtheorem{assumption}[theorem]{Assumption}

\newcommand\ubm{{\ensuremath{\bm{u}}}}
\newcommand\xbm{{\ensuremath{\bm{x}}}}
\newcommand\kbm{{\ensuremath{\bm{k}}}}
\newcommand\Kbm{{\ensuremath{\bm{K}}}}
\newcommand\fbm{{\ensuremath{\bm{f}}}}

\newcommand\ybm{{\ensuremath{\bm{y}}}}
\newcommand\vbm{{\ensuremath{\bm{v}}}}

\newcommand\hbm{{\ensuremath{\bm{h}}}}
\newcommand\Ibm{{\ensuremath{\bm{I}}}}
\newcommand\Qbm{{\ensuremath{\bm{Q}}}}
\newcommand\Dbm{{\ensuremath{\bm{D}}}}
\newcommand\zerobold{\ensuremath{\mathbf{0}}}

\title{On Improved Regret bounds in Bayesian Optimization with Gaussian noise}

\author{Jingyi Wang\thanks{Center for Applied Scientific Computing, Lawrence Livermore National Laboratory,
Livermore, CA 
  (\email{wang125@llnl.gov}).}
\and Haowei Wang \thanks{National University of Singapore}
\and Cosmin G. Petra\footnotemark[1]
\and Nai-Yuan Chiang\footnotemark[1]
}
\usepackage{amsopn}

\begin{document}
\newcommand{\Rbb}{\ensuremath{\mathbb{R} }}
\newcommand{\Pbb}{\ensuremath{\mathbb{P} }}
\newcommand{\Cbb}{\ensuremath{\mathbb{C} }}
\newcommand{\Sbb}{\ensuremath{\mathbb{S} }}
\newcommand{\Ebb}{\ensuremath{\mathbb{E} }}
\newcommand{\Vbb}{\ensuremath{\mathbb{V} }}
\newcommand{\Nbb}{\ensuremath{\mathbb{N} }}
\newcommand{\norm}[1]{\left\lVert {#1} \right\rVert}
\newcommand\epsbold{{\ensuremath{\boldsymbol{\epsilon}}}}

\maketitle

\begin{abstract}
  Bayesian optimization (BO) with Gaussian process (GP) surrogate models is a powerful black-box optimization method. 
  Acquisition functions are a critical part of a BO algorithm as they determine how the new samples are selected.
  Some of the most widely used acquisition functions include upper confidence bound (UCB) and Thompson sampling (TS). 
  The convergence analysis of BO algorithms has focused on the cumulative regret under both the Bayesian and frequentist settings for the objective. 
  In this paper, we establish new pointwise bounds on the prediction error of GP under the frequentist setting with Gaussian noise. Consequently, we prove improved convergence rates of cumulative regret bound for both GP-UCB and GP-TS.
Of note, the new prediction error bound under Gaussian noise can be applied to general BO algorithms and convergence analysis, \textit{e.g.}, the asymptotic convergence of expected improvement (EI) with noise.
\end{abstract}

\section{Introduction}
Bayesian optimization (BO) is a powerful optimization method for black-box functions~\cite{lizotte2008,jones2001taxonomy,bosurvey2023}.
A Gaussian process (GP) surrogate model is used to approximate the objective function based on iterative sampling of the data~\cite{frazier2018}.  
BO has seen enormous success in many applications such as machine learning~\cite{wu2019hyperparameter}, structural design~\cite{mathern2021}, robotics~\cite{calandra2016}, etc. 
In the classic form, BO aims to solve the optimization problem 
\begin{equation} \label{eqn:opt-prob}
 \centering
  \begin{aligned}
	  &\underset{\substack{x}\in C}{\text{maximize}} 
	  & & f(\xbm), \\
  \end{aligned}
\end{equation}
where $C\subset \Rbb^d$ represents the bound constraints on $x$, $f:\Rbb^d\to \Rbb$ is the objective function. In this paper, we consider a finite $d$.

Recent years have seen significant advances in the regret bound analysis of BO algorithms, particularly for upper confidence bound (UCB) and Thompson sampling (TS)~\cite{srinivas2009gaussian,chowdhury2017kernelized,vakili2021scalable}, which are two of the most popular acquisition functions that are simple to implement.
The theoretical analysis of the regret bound analysis is often considered either under the Bayesian setting, where $f$ is assumed to be a sample from a GP with a known kernel, or under the frequentist settings, where $f$ is assumed to be in the reproducing kernel Hilbert space (RKHS) of GP kernels.

In the seminal work of~\cite{srinivas2009gaussian}, an $\mathcal{O}(\sqrt{T\beta_T\gamma_T})$ regret bound is established for GP-UCB under both Bayesian and frequentist assumptions, where $\gamma_T$ is the maximum information gain and $\beta_T$ is an algorithmic parameter given $T$ samples.
However, the difference in $\beta_t,t=1,\dots,T$ required lead to an $\mathcal{O}(\sqrt{\gamma_T T}\log^{\frac{1}{2}}(T))$ and an $\mathcal{O}(\gamma_T\sqrt{T}\log^{\frac{3}{2}}(T))$ cumulative regret bounds for GP-UCB under Bayesian and frequentist assumptions, respectively.
Both $\beta_T$ and $\gamma_T$ play crucial roles in the rate of convergence for regret bound.
While $\beta_t$ is often chosen to guarantee the bounds on prediction error $f(\xbm)-\mu_t(\xbm)$ with high probabilities, where $\mu_t$ is the posterior mean of GP with $t$ samples, $\gamma_T$ is often dependent on the kernels of the GP.
Therefore, efforts to improve the regret bound for GP-UCB have explored on both $\beta_t$ and $\gamma_t$.

In~\cite{chowdhury2017kernelized}, the authors focused on the frequentist setting with sub-Gaussian noise and showed an $\mathcal{O}(\log^{\frac{3}{2}}(T))$ improvement in regret bound of GP-UCB by reducing $\beta_t$ and maintaining the same prediction error bound.      
Further,~\cite{chowdhury2017kernelized} showed an $\mathcal{O}(\gamma_T\sqrt{T}\log^{\frac{1}{2}}(T))$ of the regret bounds for GP-TS. Here, we adopt the finite-dimensional rate in~\cite{vakili2021information}.  
In~\cite{kandasamy2018parallelised}, under the Bayesian setting, GP-TS is shown to have the same order of cumulative regret as GP-UCB.

In~\cite{vakili2021information}, the authors proved a significantly tighter bound on $\gamma_T$ for Matérn kernels: $\mathcal{O}(T^{\frac{d}{2\nu+d}}\log^{\frac{2\nu}{2\nu+d}}(T))$, which are among the most widely used kernels. Their bound on $\gamma_T$ for squared exponential (SE) kernels are of the same order $\mathcal{O}(\log^{d+1}(T))$ as in~\cite{srinivas2009gaussian}. The authors applied their $\gamma_T$ results to both Bayesian and frequentist settings, and for both GP-UCB and GP-TS.
Variants of GP-UCB that could achieve tight regret bounds include~\cite{valko2013finite,janz2020bandit,mutny2018efficient,calandriello2019gaussian}.
In~\cite{scarlett2017lower}, lower bounds $\Omega(T^{\frac{1}{2}}\log^{\frac{d}{2}}(T))$ and $\Omega(T^{\frac{\nu+d}{2\nu+d}})$ are established for the expectation of cumulative regret for SE and Matérn kernels under the frequentist setting and Gaussian noise assumptions. Though the expectation setting is not precisely the same, these lower bounds are often used to compare with the upper regret bounds and assess the derived upper bounds.

In the literature mentioned above, the noise assumptions are not always consistent. 
The noise assumptions in the frequentist setting include zero mean uniformly bounded noise~\cite{srinivas2009gaussian}, sub-Gaussian noise~\cite{chowdhury2017kernelized}, etc.
On the other hand, the noise under Bayesian setting is often the Gaussian noise~\cite{srinivas2009gaussian,kandasamy2018parallelised}.
Notably, adopting the sub-Gaussian noise instead of the uniformly bounded one in~\cite{chowdhury2017kernelized} contributed to a smaller $\beta_t$ and subsequently improved bounds on $R_t$.  
Additionally, under the frequentist setting, the prediction error between the objective and the posterior mean $f(\xbm)-\mu_{t-1}(\xbm)$ is often derived directly to include all points in $C$ and all sample number, \textit{i.e.}, $\forall \xbm\in C, \forall t\in\Nbb$.
Hence, $f(\xbm)-\mu_{t-1}(\xbm)$ is bounded via an increasing parameter sequence $\{\beta_t\}$ and the posterior standard deviation $\sigma_{t-1}(\xbm)$, \textit{i.e.},
\begin{equation} \label{eqn:bound-1}
  \centering
  \begin{aligned}
          |f(\xbm) - \mu_{t-1}(\xbm)| \leq \beta_t^{1/2} \sigma_{t-1}(\xbm),
  \end{aligned}
\end{equation}
with high probability, where $\beta_t\to\infty$ as $t\to \infty$.
However, under the Bayesian setting with independent and identically distributed (i.i.d.) Gaussian noise, one can obtain a pointwise prediction error bound at given $\xbm\in C$ and $t\in\Nbb$, \textit{i.e.},
\begin{equation} \label{eqn:bound-bayesian}
  \centering
  \begin{aligned}
          |f(\xbm) - \mu_{t-1}(\xbm)| \leq \beta^{1/2} \sigma_{t-1}(\xbm),
  \end{aligned}
\end{equation}
with probability $1-\delta$, where $\beta$ is dependent on $\delta$.
Then,~\eqref{eqn:bound-bayesian} is lifted to the cumulative case for $\forall \xbm\in \Cbb$ and $\forall t\in \Nbb$, where $\Cbb$ is a discrete approximation of $C$, as presented in~\cite{srinivas2009gaussian}.
The lack of pointwise prediction error bound in the frequentist setting not only hinders the asymptotic analysis, as $\beta_t$ is unbounded as $t\to\infty$,  but also prevents the superior cumulative regret bound achieved in the Bayesian setting from being applied.

Our contribution in this paper can be summarized as follows. 
Firstly, we adopt the same i.i.d Gaussian noise assumption as in~\cite{srinivas2009gaussian,kandasamy2018parallelised,scarlett2017lower} and prove the pointwise prediction error bound but in the frequentist setting. 
Second, we prove improved cumulative regret bound of GP-UCB in the frequentist setting for two of the most commonly used kernels: the squared exponential (SE) kernel and Matérn kernel. For a fair comparison, we apply the same rate for $\gamma_T$ from~\cite{vakili2021information} to all selected literature and present the results in Table~\ref{tbl:regretbound}. To the best of our knowledge, our bounds improve upon the state-of-the-art and is only $\log(T)$ worse than the lower bounds established for the expectation of regret in~\cite{scarlett2017lower}.
We also close the current regret bound gap for GP-UCB in the Bayesian and frequentist settings.
Third, we prove improved cumulative regret upper bound of GP-TS for SE and Matérn kernels. Basically, we aim to complement the current theoretical results on cumulative regret bounds. All results are summarized in Table 1 for easy illustration and comparison.

\begin{table}[H]
\centering
\resizebox{\textwidth}{!}{
\begin{tabular}{|c|c|c|c|c|}
  \hline 
   & GP-UCB, SE &GP-UCB, Matérn & GP-TS, SE & GP-TS, Matérn  \\
 \hline
  \makecell{$R_T^1$~\cite{srinivas2009gaussian}}   & $\mathcal{O}(\sqrt{T}\log^{\frac{2d+5}{2}}(T) )$  & $\mathcal{O}(T^\frac{\nu+3d/2}{2\nu+d}\log^{\frac{10\nu+3d}{4\nu+2d}}(T))$ & -  & - \\
 \makecell{$R_T^2$~\cite{chowdhury2017kernelized}}  & $\mathcal{O}(\sqrt{T}\log^{1+d}(T) )$   &$\mathcal{O}(T^\frac{\nu+3d/2}{2\nu+d}\log^{\frac{2\nu}{2\nu+d}}(T))$ &$\mathcal{O}(\sqrt{T}\log^{d+\frac{3}{2}}(T))$ &$\mathcal{O}(T^\frac{\nu+3d/2}{2\nu+d}\log^{\frac{6\nu+2d}{4\nu+2d}}(T))$ \\
 \makecell{$R_T^{3}$~\cite{vakili2021information}}       &$\mathcal{O}(\sqrt{T}\log^{\frac{d+5}{2}}(T) )$   &$\mathcal{O}(T^{\frac{\nu+d}{2\nu+d}}\log^{\frac{10\nu+4d}{4\nu+2d}}(T))$ &same as~\cite{chowdhury2017kernelized} & same as~\cite{chowdhury2017kernelized} \\ 
  \makecell{$R_T$ {\color{red}(Ours)} } &$\mathcal{O}(\sqrt{T}\log^{\frac{d+2}{2}}(T))$ &$\mathcal{O}(T^{\frac{\nu+d}{2\nu+d}}\log^{\frac{4\nu+d}{4\nu+2d}}(T))$ & $\mathcal{O}(\sqrt{T}\log^{\frac{d+3}{2}}(T))$ &$\mathcal{O}(T^{\frac{\nu+d}{2\nu+d}}\log^{\frac{3\nu+d}{2\nu+d}}(T) )$\\
  \hline
  \makecell{$\Ebb[R_T]$~\cite{scarlett2017lower}} & $\Omega(\sqrt{T}\log^{\frac{d}{2}}(T))$   & $\Omega(T^{\frac{\nu+d}{2\nu+d}})$ & - & - \\
\hline
\end{tabular}
}
\caption{\em Cumulative regret bounds in the frequentist setting using $\gamma_T$ from~\cite{vakili2021information}. That is, $\gamma_T=\mathcal{O}(\log^{d+1}(T))$ for SE kernel and $\gamma_T=\mathcal{O}(T^{\frac{d}{2\nu+d}}\log^{\frac{2\nu}{2\nu+d}}(T))$ for Matérn kernel. For $R_T^{3}$, the GP-UCB results are obtained by SupKernelUCB~\cite{valko2013finite}, not the basic UCB. 
The GP-TS regret bound for $R_T^{3}$ is directly taken from~\cite{chowdhury2017kernelized} and therefore the same.
The last row are the results reported in this paper, and are clearly improved. Our results are also closer to the lower bound established in \cite{scarlett2017lower}}
 
\label{tbl:regretbound}
\end{table}

This paper is organized as follows. In section~\ref{se:bo}, we describe the GP-UCB and GP-TS algorithms. In section~\ref{se:analysis}, we present the pointwise prediction error bound under the frequentist assumptions and Gaussian noise. Then, we prove its extension to a discrete set $\Cbb$ and any sample number $\forall t\in\Nbb$.
In section~\ref{se:regret}, we establish improved cumulative regret bounds for GP-UCB, which now have the same convergence rates as the Bayesian setting.
In section~\ref{se:thompson}, we prove improved cumulative regret bound for GP-TS.
Conclusions are offered in section~\ref{se:conclusion}.

\section{Background}\label{se:bo}
In this section, we present background on BO, GP-UCB, GP-TS and cumulative regret.
A GP surrogate model takes the prior distribution to be multivariate Gaussian distribution, based on the existing samples. 
BO with GP surrogate models often uses an acquisition function to determine the next sample, which balances the trade-off between exploration and exploitation.  

\subsection{Gaussian process}\label{se:gp}
Consider a zero mean GP with the kernel (\textit{i.e.}, covariance function) $k(\xbm,\xbm'):\Rbb^d\times\Rbb^d\to\Rbb$, denoted as $GP(0,k(\xbm,\xbm'))$. 
We make the standard assumption in literature~\cite{srinivas2009gaussian} that $k(\xbm,\xbm')\leq 1$ and $k(\xbm,\xbm)=1$, $\forall \xbm,\xbm'\in C$ for $f$.
Choices of the kernels include the SE function, Matérn functions, etc.
At $\xbm_t\in C$, we admit observation noise $\epsilon_t$ for $f$ so that the observed function value is 
$y_t = f(\xbm_t)+\epsilon_t$.  
Thus, the prior distribution on 
$\xbm_{1:t}=[\xbm_1,\dots,\xbm_t]^T$, $\fbm_{1:t}=[f(\xbm_1),\dots,f(\xbm_t)]^T$ and $\ybm_{1:t}=[y_1,\dots,y_t]^T$
is $\fbm_{1:t}\sim\mathcal{N} (\mathbf{0}, \Kbm_t)$, 
where $\mathcal{N}$ denotes the normal distribution, 
and $\Kbm_t= [k(\xbm_1,\xbm_1),\dots,k(\xbm_1,\xbm_t); \dots; k(\xbm_t,\xbm_1),\dots,k(\xbm_t,\xbm_t)]$. 
Without losing generality, we assume the covariance matrix $\Kbm_t$ is symmetric positive semi-definite.

We assume that the noise follow an independent zero mean Gaussian with variance $\sigma_{\epsilon}^2$, \textit{i.e.}, $\epsilon_t\sim\mathcal{N}(0,\sigma_{\epsilon}^2)$. The BO algorithm often does not have knowledge of $\sigma_{\epsilon}$ and instead assumes a standard deviation of $\sigma$.
In section~\ref{se:analysis}, we allow $\sigma$ and $\sigma_{\epsilon}$ to be different.
It is possible for $\epsilon_t$ to be relaxed to have a more general assumption (see Remark~\ref{remark:generalnoise} for examples). 
Denote $\epsbold_{1:t}=[\epsilon_1,\dots,\epsilon_t]^T$.
The posterior distribution of  $f(\xbm) | \xbm_{1:t},\ybm_{1:t} \sim \mathcal{N} (\mu(\xbm),\sigma^2(\xbm))$ can  be inferred using Bayes' rule. 
The posterior mean $\mu$ and variance  $\sigma^2$ are  
\begin{equation} \label{eqn:GP-post}
 \centering
  \begin{aligned}
  &\mu_t(\xbm)\ =\ \kbm_t(\xbm) \left(\Kbm_t+ \sigma^2 \Ibm\right)^{-1} \ybm_{1:t} \\
  &\sigma^2_t(\xbm)\ =\
k(\xbm,\xbm)-\kbm_t(\xbm)^T \left(\Kbm_t+\sigma^2 \Ibm\right)^{-1}\kbm_t(\xbm)\ ,
\end{aligned}
\end{equation}
where $\kbm_t(\xbm)= [k(\xbm_1,\xbm),\dots,k(\xbm_t,\xbm)]^T$. 
\subsection{Upper confidence bound}\label{se:ucb}
The UCB acquisition function with $t$ samples is defined by  
\begin{equation} \label{eqn:UCB}
 \centering
  \begin{aligned}
       UCB_{t}(\xbm) =  \mu_{t}(\xbm)+\beta^{1/2}_{t+1}\sigma_{t}(\xbm), 
  \end{aligned}
\end{equation}
where $\beta_{t+1}\geq 0$ is an algorithmic parameter that blances the trade-off between exploration and exploitation and plays an important role in theoretical regret analysis. 
Indeed, $\beta_t$ is often chosen based on cumulative regret bounds~\cite{srinivas2009gaussian}. 

The GP-UCB algorithm chooses its next sample via
\begin{equation} \label{eqn:acquisition-ucb}
 \centering
  \begin{aligned}
      \xbm_{t+1} = \underset{\substack{\xbm\in C}}{\text{argmax}}  UCB_t (\xbm).
  \end{aligned}
\end{equation}
In order to solve~\eqref{eqn:acquisition-ucb}, optimization algorithms including L-BFGS or random 
search can be used.
The GP-UCB algorithm is given in Algorithm~\ref{alg:boucb}, where a stopping criterion can be added.
\begin{algorithm}[H]
 \caption{GP-UCB}\label{alg:boucb}
  \begin{algorithmic}[1]
	  \STATE{Choose $\mu_0=0$, $k(\cdot,\cdot)$, $\alpha$, and $T_0$ initial samples $\xbm_i, i=0,\dots,T_0$. Observe $y_i$.} 
	  \STATE{Train the Gaussian process surrogate model for $f$ on the initial samples.}
  \FOR{$t=T_0,T_0+1,\dots$}
	  \STATE{Find $\xbm_{t+1}$ based on~\eqref{eqn:acquisition-ucb}.}
	  \STATE{Observe $y_{t+1}$. \;}
	  \STATE{Train the surrogate model with the addition of $\xbm_{t+1}$ and $y_{t+1}$.\;}
	  \IF {Stopping criterion satisfied} 
              \STATE{Exit}
	  \ENDIF
  \ENDFOR
  \end{algorithmic}
\end{algorithm}

\subsection{Thompson Sampling}\label{se:tsalg}
Another widely adopted acqusition function is Thompson sampling~\cite{agrawal2013thompson}. In this paper, we adopt the design in~\cite{chowdhury2017kernelized} where a scaling parameter $\nu_t$ and a discretization $\Cbb_t$ are employed to find the next sample.
Specifically, at each iteration, GP-TS samples a function $f_t$ from the GP with mean $\mu_{t-1}(\cdot)$ 
and covariance of $\nu_t k_{t-1}(\cdot,\cdot)$, where the scaling parameter $\nu_t$ is highly correlated to $\beta_t$ in the GP-UCB algorithm. The discretization $\Cbb_t$ used in the analysis of CP-UCB is incorporated into the algorithm as a decision set and the next sample is chosen from within $\Cbb_t$. 
The prior of this GP-TS is nonparameteric, as stated in~\cite{chowdhury2017kernelized}.
Specifically, for consistency in comparison to~\cite{chowdhury2017kernelized}, $\nu_t$ is chosen to be 
 \begin{equation} \label{def:nut}
  \centering
  \begin{aligned}
   \nu_t^{1/2} = B+ \sqrt{2\log(\frac{\pi^2 t^2}{3}/\delta) + 2d \log( 1+rL t^2)}. 
   \end{aligned}
 \end{equation} 
That is, $\nu_t$ satisfies Lemma~\ref{lem:f-discret-bound-RKHS} in the analysis for GP-UCB with $\delta/2$ and $\pi_t=\frac{\pi^2t^2}{6}$.
The GP-TS algorithm is given in Algorithm~\ref{alg:bots}.
\begin{algorithm}[H]
 \caption{GP-TS}\label{alg:bots}
  \begin{algorithmic}[1]
	  \STATE{Choose $\mu_0=0$, $k(\cdot,\cdot)$, $B$, $\sigma$, and $T_0$ initial samples $\xbm_i, i=0,\dots,T_0$. Observe $y_i$.} 
	  \STATE{Train the Gaussian process surrogate model for $f$ on the initial samples.}
  \FOR{$t=T_0,T_0+1,\dots$}
          \STATE{Set $\nu_t$ according to~\eqref{def:nut}.}
          \STATE{Choose the discretization $\Cbb_t$ based on~\eqref{eqn:disc-size-1}.}
          \STATE{Sample $f_t(\cdot)$ from $GP(\mu_{t-1}(\cdot), \nu_t k_{t-1}(\cdot,\cdot))$}
	  \STATE{Find $\xbm_{t+1} = \underset{\substack{\xbm\in \Cbb_t}}{\text{argmax}}  f_t (\xbm)$.}
	  \STATE{Observe $y_{t+1}$. \;}
	  \STATE{Train the surrogate model with the addition of $\xbm_{t+1}$ and $y_{t+1}$.\;}
	  \IF {Stopping criterion satisfied} 
              \STATE{Exit}
	  \ENDIF
  \ENDFOR
  \end{algorithmic}
\end{algorithm}
For simplicity of analysis, we assume in the case of GP-TS, $\sigma_{\epsilon}\leq \sigma$ in section~\ref{se:thompson}.

\subsection{Cumulative regret}\label{se:other}
In BO literature, cumulative regret $R_T$, $T$ being the number of samples, is defined as $R_T = \sum_{t=1}^T r_t$, where $r_t$ is the instantaneous regret. 
For GP-UCB and GP-TS, $r_t$ is defined by $r_t = f(\xbm_t)-f(\xbm^*)$.
An algorithm with no regret refers to $\lim_{T\to\infty} R_T/T=0$. 
It is well-known that GP-UCB and GP-TS are no-regret algorithms.
The regret bounds are often proven via information theory results that are well-established in literature.
The maximum information gain is defined below.
\begin{definition}\label{def:infogain}
  Given $\xbm_{1:t}=[\xbm_1,\dots,\xbm_t]$ and $A_t:=\ybm_{1:t}=[\ybm_1,\dots,\ybm_t]$, the mutual information between $f$ and $\ybm_{1:t}$ is $I(\ybm_A;f)=H(\ybm_A)-H(\ybm_A|f)$, where $H$ is the entropy. The maximum information gain $\gamma_T$ after $T$ samples is $\gamma_T = \max_{A\subset C,|A|=T} I(\ybm_A;\fbm_A)$. 
\end{definition}
The properties of union bound, or Boole's inequality, is used often in the analysis. For a countable set of events $A_1,A_2,\cdots$, we have 
\begin{equation} \label{def:union}
 \centering
  \begin{aligned}
    P(\bigcup_{i=1}^{\infty} A_i) \leq \sum_{i=1}^{\infty} P(A_i).
  \end{aligned}
\end{equation}

\subsection{Basic assumptions}
For ease of reference, we list the assumptions used in our paper below, all of which are common in BO literature.
\begin{assumption}\label{assp:rkhs}
The objective function $f$ is in the RKHS $\mathcal{H}_k$ associated with the kernel $k(\xbm,\xbm')\leq 1$ and $k(\xbm,\xbm)=1$. Further, the RKHS norm of $f$ is bounded $\norm{f}_{\mathcal{H}_k}\leq B$ for some $B\geq 1$. 
\end{assumption}
\begin{assumption}\label{assp:lipschitz}
The objective function $f$ is Lipschitz continuous with Lipschitz constant $L$ for 2-norm $\norm{\cdot}$. 
\end{assumption}
Since our results are given with probabilities, Assumption~\ref{assp:lipschitz} can be further generalized to $f$ being Lipschitz with a high probability, possibly decaying with $L$.
The discrete set assumption on $C$ is given below, used in an intermediate step in cumulative regret analysis. 
\begin{assumption}\label{assp:discrete}
  The set $C$ is discrete and has a finite cardinality $|C|$.
\end{assumption}
Without losing generality, the general bound constraint assumption is given below.
\begin{assumption}\label{assp:constraint}
   The set $C\subseteq [0,r]^d$ is convex and compact for $r>0$.
\end{assumption}
Finally, the Gaussian noise assumption is given below.
\begin{assumption}\label{assp:gaussiannoise}
  The observation noise are i.i.d. random samples from a zero mean Gaussian, \textit{i.e.}, $\epsilon_t\sim\mathcal{N}(0,\sigma_{\epsilon}^2),\sigma_{\epsilon}>0$ for all $t$.
\end{assumption}

\section{Prediction error bound with Gaussian noise}\label{se:analysis}
In this section, we present our result on the prediction error $f(\xbm)-\mu_t(\xbm)$. In general, we focus on $\sigma_t(\xbm) >0$ as the results for $\sigma_t(\xbm)=0$ are trivial.
First, we state some obvious properties of $\phi$, $\Phi$ and $\tau$.
\begin{lemma}\label{lem:phi}
The density functions satisfy $0< \phi(x)\leq \phi(0)<0.4, \Phi(x)\in(0,1)$, 
for any $x\in\Rbb$. 
 Given a random variable $\xi$ that follows the standard normal distribution: $\xi\sim\mathcal{N}(0,1)$. Then, the probability of $r>c, c>0$ satisfies $P\{r>c\}\leq \frac{1}{2}e^{-c^2/2}$. 
 Similarly, for $c<0$, $P\{r<c\}\leq \frac{1}{2}e^{-c^2/2}$. 
\end{lemma}
The last statement in Lemma~\ref{lem:phi} is a well-known result for standard normal distribution (\textit{e.g.}, see proof of Lemma 5.1 in~\cite{srinivas2009gaussian}).

\subsection{Pointwise prediction error}\label{se:gaussian-noise}
We first note that the posterior mean, $\mu_t(\xbm)\ =\ \kbm_t(\xbm) \left(\Kbm_t+ \sigma^2 \Ibm\right)^{-1} \ybm_{1:t}$ ~\eqref{eqn:GP-post} can be viewed as a linear combination of the observations $\ybm_{1:t}$, which is the sum of function values and observed noises, \textit{i.e.}, $\ybm_{1:t} = \fbm_{1:t}+\epsbold_{1:t}$.  
Assumptions~\ref{assp:rkhs} and~\ref{assp:gaussiannoise} hold in this section.
Under Assumption~\ref{assp:rkhs}, the posterior mean generated by $\fbm_{1:t}$, i.e. $\kbm_t^T(\xbm)(\Kbm_t+\sigma^2 \Ibm)^{-1} \fbm_{1:t}$ has been shown to be bounded by the posterior standard deviation, as stated in the Lemma below.
\begin{lemma}\label{lem:rkhs-f-1}
Suppose Assumption~\ref{assp:rkhs} holds. For any given $\xbm\in C$ and $t\in\Nbb$,
\begin{equation} \label{eqn:rkhs-f-1}
  \centering
  \begin{aligned}
          |f(\xbm) - \kbm_t^T(\xbm)(\Kbm_t+\sigma^2 \Ibm)^{-1} \fbm_{1:t}| \leq B\sigma_t({\xbm}), 
  \end{aligned}
\end{equation}
  where $B$ is the norm bound in Assumption~\ref{assp:rkhs}. 
\end{lemma}
The proof of Lemma~\ref{lem:rkhs-f-1} can be found in~\cite{chowdhury2017kernelized,bull2011convergence,wang2023regret}.

To obtain a bound on the prediction error $f(\xbm)-\mu_t({\xbm})$, we focus on the posterior mean induced by the noise, denoted as
\begin{equation} \label{eqn:rkhs-e-1}
  \centering
  \begin{aligned}
  e_t=\kbm_t(\xbm)^T (\Kbm_t+\sigma^2\Ibm)^{-1}\epsbold_{1:t},
  \end{aligned}
\end{equation}
and its relationship with the posterior variance~\eqref{eqn:GP-post}.
Define
\begin{equation} \label{eqn:rkhs-h}
  \centering
  \begin{aligned}
   \hbm_t(\xbm)=(\Kbm_t+\sigma^2\Ibm)^{-1}\kbm_t(\xbm). 
  \end{aligned}
\end{equation}
We can view~\eqref{eqn:rkhs-e-1} as a linear combination of $\epsbold_{1:t}$, and GP itself as a linear smoother~\cite{williams2006gaussian}. Our goal is to bound $e_t$ with a finite constant multiplying $\sigma_{t}(\xbm)$.

\begin{theorem}\label{thm:rkhs-f-bound-finite}
  Given $\delta\in(0,1)$, let $\beta^{1/2}= B+\frac{\sigma_{\epsilon}}{\sigma}\sqrt{2\log(1/\delta)}$.
  Then, with probability $\geq 1-\delta$, 
\begin{equation} \label{eqn:rkhs-f-bound-gaussian-1}
  \centering
  \begin{aligned}
          |f(\xbm) - \mu_t(\xbm)| \leq  \beta^{1/2}\sigma_t(\xbm),
  \end{aligned}
\end{equation}
   at given $\xbm\in C$ and $t\in\Nbb$.
\end{theorem}
\begin{proof}
   We first prove that for a finite $\sigma>0$ and $t\in\Nbb$, 
  $\norm{\hbm_t(\xbm)}\leq \frac{1}{\sigma} \sigma_{t}(\xbm)$.
   We shall repeatedly use the fact that the covariance matrix $\Kbm_t$ is symmetric and positive semi-definite. Specifically, given $t$ samples and $\xbm\in C$, the covariance matrix of $\fbm_{1:t}$ and $f(\xbm)$  with zero noise 
 \begin{equation} \label{eqn:rkhs-e-pf-1}
  \centering
  \begin{aligned}
    \begin{bmatrix*}
               \Kbm_t &\kbm_t\\
               \kbm_t^T &k(\xbm,\xbm)
      \end{bmatrix*}  
  \end{aligned}
\end{equation}
   is symmetric and positive semi-definite.  On the other hand, the covariance matrix with a Gaussian noise (assumed) $\mathcal{N}(0,\sigma^2)$ leads to  
 \begin{equation} \label{eqn:rkhs-e-pf-2}
  \centering
  \begin{aligned}
    \begin{bmatrix}
               \Kbm_t +\sigma^2\Ibm & \kbm_t\\
               \kbm_t^T & k(\xbm,\xbm)
      \end{bmatrix},  
  \end{aligned}
\end{equation}
  which is positive definite for $\forall \sigma>0$. Since $\Kbm_t$ is symmetric and positive semi-definite, we have the spectral decomposition~\cite{demmel1997applied} 
 \begin{equation} \label{eqn:rkhs-e-pf-3}
  \centering
  \begin{aligned}
               \Kbm_t = \Qbm^T_t \Dbm_t^K \Qbm_t, 
  \end{aligned}
\end{equation}
  where $\Qbm_t$ is an orthogonal $t\times t$ matrix and $\Dbm_t^K$ is a diagonal matrix whose entries are the eigenvalues of $\Kbm_t$. 
  For simplicity, denote the eigenvalues of $\Kbm_t$ as $\lambda_i,i=1\dots,t$. 
  Without losing genearlity, let $(\Dbm_t^K)_i = \lambda_i \geq 0$.

  Using the spetral decomposition~\eqref{eqn:rkhs-e-pf-3}, we have
 \begin{equation} \label{eqn:rkhs-e-pf-4}
  \centering
  \begin{aligned}
   \Kbm_t+\sigma^2\Ibm= \Qbm^T_t \Dbm_t^{K1} \Qbm_t,  \ (\Kbm_t+\sigma^2\Ibm)^{-1}= \Qbm^T_t (\Dbm_t^{K1})^{-1} \Qbm_t, 
  \end{aligned}
  \end{equation}
where 
 \begin{equation} \label{eqn:rkhs-e-pf-5}
  \centering
  \begin{aligned}
    \Dbm_t^{K1} =  (\Dbm_t^K+\sigma^2\Ibm), (\Dbm_t^{K1})_i = \lambda_i +\sigma^2,
  \end{aligned}
  \end{equation} 
  for $ i=1,\dots,t$. 
  Next, consider the matrix $(\sigma^2 (\Kbm_t+\sigma^2 I)^{-1}+\Ibm)^{-1}(\Kbm_t+\sigma^2 I)$.
Applying~\eqref{eqn:rkhs-e-pf-4}, we can write 
 \begin{equation} \label{eqn:rkhs-e-pf-6}
  \centering
  \begin{aligned}
(\sigma^2 (\Kbm_t+&\sigma^2 I)^{-1}+\Ibm)^{-1}(\Kbm_t+\sigma^2 I)= (\sigma^2 \Qbm_t^T (\Dbm_t^{K1})^{-1}\Qbm_t +\Ibm)^{-1} \Qbm_t^T\Dbm^{K1}_t\Qbm_t\\
       =& (\Qbm^T_t \Dbm^{K2}_t \Qbm_t)^{-1} \Qbm^T_t\Dbm^{K1}_t\Qbm_t = \Qbm^T_t (\Dbm^{K2}_t)^{-1}\Dbm^{K1}_t\Qbm_t =\Qbm^T_t \Dbm^{K3}_t\Qbm_t, 
  \end{aligned}
\end{equation}
  where 
 \begin{equation} \label{eqn:rkhs-e-pf-7}
  \centering
  \begin{aligned}
        \Dbm^{K2}_t = \sigma^2(\Dbm_t^{K1})^{-1}+\Ibm, \Dbm_t^{K3} = (\Dbm_t^{K2})^{-1}\Dbm_t^{K1}.
  \end{aligned}
\end{equation}
  Thus, $(\sigma^2 (\Kbm_t+\sigma^2 I)^{-1}+\Ibm)^{-1}(\Kbm_t+\sigma^2 I)$ is symmetric and positive definite.
   The components of the $\Dbm_t^{K2}$ and $\Dbm_t^{K3}$ are
 \begin{equation} \label{eqn:rkhs-e-pf-8}
  \centering
  \begin{aligned}
   (D_t^{K2})_i = \frac{\sigma^2}{\lambda_i+\sigma^2}+1, (D_t^{K3})_i = \frac{(\lambda_i+\sigma^2)^2}{\lambda_i+2\sigma^2}.
  \end{aligned}
\end{equation}
Thus,  
 \begin{equation} \label{eqn:rkhs-e-pf-9}
  \centering
  \begin{aligned}
  (D_t^{K3})_i = \frac{ \lambda_i^2+2\lambda_i\sigma^2+\sigma^4}{\lambda_i+2\sigma^2}=\lambda_i+\frac{\sigma^4}{\lambda_i+2\sigma^2}> \lambda_i.
  \end{aligned}
\end{equation}
 From~\eqref{eqn:rkhs-e-pf-6} and~\eqref{eqn:rkhs-e-pf-9},  for vector $\forall \vbm \in\Rbb^t$ and $\vbm\neq \zerobold$
 \begin{equation} \label{eqn:rkhs-e-pf-10}
  \centering
  \begin{aligned}
          \vbm^T \Kbm_t \vbm < \vbm^T (\sigma^2 (\Kbm_t+\sigma^2 I)^{-1}+\Ibm)^{-1}(\Kbm_t+\sigma^2 I)\vbm.
  \end{aligned}
\end{equation}
  Since~\eqref{eqn:rkhs-e-pf-1} is symmetric and positive semi-definite, we can construct vector $[\vbm,u]\in \Rbb^{t+1}$, for $\forall \vbm\in\Rbb^t$, $u\in\Rbb$ and $u\vbm\neq \zerobold$. Then, 
 \begin{equation} \label{eqn:rkhs-e-pf-11}
  \centering
  \begin{aligned}
  \begin{bmatrix*}
               \vbm\\
               u
      \end{bmatrix*}^T 
    \begin{bmatrix*}
               \Kbm_t &\kbm_t\\
               \kbm_t^T &k(\xbm,\xbm)
      \end{bmatrix*} 
 \begin{bmatrix*}
               \vbm\\
               u 
      \end{bmatrix*}= \vbm^T \Kbm_t\vbm + 2u\vbm^T\kbm+u^2 \geq 0
  \end{aligned}
\end{equation}
Replacing $\Kbm_t$ in~\eqref{eqn:rkhs-e-pf-11}, by~\eqref{eqn:rkhs-e-pf-10} 
 \begin{equation} \label{eqn:rkhs-e-pf-12}
  \centering
  \begin{aligned}
  &\begin{bmatrix*}
               \vbm\\
               u
      \end{bmatrix*}^T 
    \begin{bmatrix*}
              (\sigma^2 (\Kbm_t+\sigma^2 I)^{-1}+\Ibm)^{-1}(\Kbm_t+\sigma^2 I) &\kbm_t\\
               \kbm_t^T &k(\xbm,\xbm)
      \end{bmatrix*} 
 \begin{bmatrix*}
               \vbm\\
               u 
      \end{bmatrix*} \\
       &= \vbm^T (\sigma^2 (\Kbm_t+\sigma^2 I)^{-1}+\Ibm)^{-1}(\Kbm_t+\sigma^2 I) \vbm + 2u\vbm^T\kbm+u^2\\ 
       &> \vbm^T  \Kbm_t\vbm + 2u\vbm^T\kbm+u^2 \geq 0.
  \end{aligned}
\end{equation}
Therefore, the matrix in~\eqref{eqn:rkhs-e-pf-12} is symmetric and positive definite, and thus its Schur complement is positive, \textit{i.e.},  
 \begin{equation} \label{eqn:rkhs-e-pf-13}
  \centering
  \begin{aligned}
    1- \kbm_t^T (\Kbm_t+\sigma^2 I)^{-1} (\sigma^2 (\Kbm_t+\sigma^2 I)^{-1}+\Ibm)\kbm_t> 0.
  \end{aligned}
\end{equation}
Simple rearrangement of~\eqref{eqn:rkhs-e-pf-13} leads to 
\begin{equation} \label{eqn:rkhs-e-pf-14}
  \centering
  \begin{aligned}
     \kbm_t^T (\Kbm_t+\sigma^2 I)^{-1}\kbm_t + \kbm_t^T \sigma^2 (\Kbm_t+\sigma^2 I)^{-1} (\Kbm_t+\sigma^2 I)^{-1}\kbm_t < 1.
  \end{aligned}
\end{equation}
Or equivalently, 
\begin{equation} \label{eqn:rkhs-e-pf-15}
  \centering
  \begin{aligned}
      \norm{\hbm_t(\xbm)} \leq \frac{1}{\sigma} \sqrt{1- \kbm_t^T (\Kbm_t+\sigma^2 I)^{-1}\kbm_t}=\frac{1}{\sigma}\sigma_t(\xbm).
  \end{aligned}
\end{equation}
  We can now bound the prediction error $f(\xbm)-\mu_t(\xbm)$ under the Gaussian noise assumption.
  From~\eqref{eqn:rkhs-e-1}, $e_t$ is a linear combination of $\epsbold_{1:t}$.
  From the properties of the expectation and variance operator, 
\begin{equation} \label{eqn:rkhs-f-bound-finite-pf-1}
  \centering
  \begin{aligned}
     \Ebb[e_t(\xbm)] = \sum_{i=1}^t (h_t)_i(\xbm) \Ebb[\epsbold_t] = 0, \ \Vbb[e_t(\xbm)] = \sum_{i=1}^t (h_t)^2_i(\xbm) \Vbb[\epsilon_i] = \norm{\hbm_t(\xbm)}^2 \sigma_{\epsilon}^2.   
  \end{aligned}
\end{equation}
  By~\eqref{eqn:rkhs-e-pf-15},
\begin{equation} \label{eqn:rkhs-f-bound-finite-pf-2}
  \centering
  \begin{aligned}
            \Vbb[e_t(\xbm)]  \leq  \frac{\sigma_{\epsilon}^2}{\sigma^2}\sigma_t^2(\xbm).   
  \end{aligned}
\end{equation}
  Since $\epsilon_i$ are i.i.d Gaussian samples, $e_t$ is also a Gaussian sample, where 
\begin{equation} \label{eqn:rkhs-f-bound-finite-pf-3}
  \centering
  \begin{aligned}
      e_t(\xbm)\sim\mathcal{N}(0, (\sigma_t^e)^2(\xbm)), \ \sigma_t^e=\Vbb[e_t(\xbm)]\leq \frac{\sigma_{\epsilon}}{\sigma}\sigma_t(\xbm).
   \end{aligned}
\end{equation}
  Define $\alpha=2\log(1/\delta)$. By Lemma~\ref{lem:phi},
\begin{equation} \label{eqn:rkhs-f-bound-gaussian-pf-1}
    \centering
    \begin{aligned}
          \Pbb\{e_t(\xbm) > \alpha^{1/2}\sigma_t^e(\xbm)\} \leq   \frac{1}{2} e^{-\alpha/2}\leq \frac{1}{2}\delta.
    \end{aligned}
  \end{equation}
  Similarly,
\begin{equation} \label{eqn:rkhs-f-bound-gaussian-pf-2}
    \centering
    \begin{aligned}
          \Pbb\{e_t(\xbm) < -\alpha^{1/2}\sigma_t^e(\xbm)\} \leq \frac{1}{2} e^{-\alpha/2}\leq \frac{1}{2}\delta.
    \end{aligned}
  \end{equation}
  Therefore, 
\begin{equation} \label{eqn:rkhs-f-bound-gaussian-pf-3}
    \centering
    \begin{aligned}
        \Pbb\{|e_t(\xbm)| > \alpha^{1/2}\frac{\sigma_{\epsilon}}{\sigma}\sigma_t(\xbm)\} \leq \Pbb\{|e_t(\xbm)| > \alpha^{1/2}\sigma_t^e(\xbm)\}  \leq  \delta.
    \end{aligned}
  \end{equation}
  From Lemma~\ref{lem:rkhs-f-1},
\begin{equation} \label{eqn:rkhs-f-bound-gaussian-pf-4}
  \centering
  \begin{aligned}
       |f(\xbm)-\mu_t(\xbm)| =& |f(\xbm) - \kbm_t^T(\xbm)(\Kbm_t+\sigma^2 \Ibm)^{-1} \ybm_{1:t}|\\
        \leq& |f(\xbm)-\hbm_t^T(\xbm)\fbm_{1:t}|+|e_t(\xbm)|\\
        \leq& (B+\frac{\sigma_{\epsilon}}{\sigma}\alpha^{1/2}) \sigma_t(\xbm) = \beta^{1/2} \sigma_t(\xbm),
   \end{aligned}
  \end{equation} 
  with probability $\geq 1-\delta$. 

\end{proof}

\begin{remark}\label{remark:generalnoise}
 The Gaussian noise assumption could be relaxed  in Theorem~\ref{thm:rkhs-f-bound-finite}. 
  A more general assumption can be the following. 
\begin{assumption}\label{assp:generalnoise}
  The observation noise $\epsilon_i$ are i.i.d. samples from a distribution that satisfies: $\Ebb[\epsilon_i] = 0$ and $\Vbb[\epsilon_i]\leq {\sigma_0}^2$ for all $i=1,\dots,t$ and some $\sigma_0>0$.  
   Moreover, the linear combination $e_t = \sum_{i=1}^t a_i\epsilon_i$, where $a_i\in \Rbb$ satisfies the following: 
   for $\forall \delta\in(0,1]$, there exists an invertible function $p_{\epsilon}(\delta)$ so that $e_t \leq  p_{\epsilon}(\delta) \sqrt{\Vbb[\epsilon_i]}$, with probability $\geq 1-\delta$.
\end{assumption}
 Unlike the Gaussian distribution, the linear combination of i.i.d. samples of another distribution is not guaranteed to follow the same distribution. One might also be able to generalize the Gaussian noise assumption from the central limit theorem  perspective.
\end{remark}

\subsection{Prediction error for general set}
In this section, we extend results at given $\xbm\in C$ and $t\in\Nbb$ to $\forall \xbm\in \Cbb, \forall t\in \Nbb$, where $\Cbb$ is a discretization of $C$, that are necessary for cumulative regret bounds. 
The proof techniques bear similarity to those in~\cite{srinivas2009gaussian} under the Bayesian setting. We note that for a finite $t$, $\Cbb$ remains a subset of $C$. 
In this section, Assumption~\ref{assp:rkhs} and Assumption~\ref{assp:gaussiannoise} are considered valid. 
Consider first the discrete bound set with Assumption~\ref{assp:discrete}. 
\begin{lemma}\label{lem:f-bound-RKHS-1}
   Given $\delta\in(0,1)$, let $\beta_t^{1/2}= B+ c_\sigma\sqrt{2\log(|C|\pi_t/\delta)}$ where $\sum_{t=1}^{\infty} \frac{1}{\pi_t}=1$ and $c_{\sigma}=\frac{\sigma_{\epsilon}}{\sigma}$. Under Assumption~\ref{assp:discrete},
      \begin{equation} \label{eqn:EI-bound-RKHS-1}
  \centering
  \begin{aligned}
        |f(\xbm)-\mu_{t-1}(\xbm)|\leq \beta_t^{1/2} \sigma_{t-1}(\xbm),
   \end{aligned}
  \end{equation} 
  with probability $\geq 1-\delta$, for $\forall \xbm\in C$ and $t\in\Nbb$.
\end{lemma}
\begin{proof}
    Following the proof of Theorem~\ref{thm:rkhs-f-bound-finite}, at $\xbm \in C$ and $t\in\Nbb$, $e_{t-1}(\xbm)\sim\mathcal{N}(0,\sigma_{t-1}^e(\xbm))$, where $\sigma_{t-1}^e(\xbm)\leq \frac{\sigma_{\epsilon}}{\sigma}\sigma_{t-1}(\xbm)$. 
  Define $\alpha_t^{1/2}=2\log(|C|\pi_t/\delta)$, similar to~\eqref{eqn:rkhs-f-bound-gaussian-pf-1},
  \begin{equation} \label{eqn:EI-bound-RKHS-pf-1}
    \centering
    \begin{aligned}
          \Pbb\{e_{t-1}(\xbm) > \alpha_t^{1/2}\sigma_{t-1}^e(\xbm)\} \leq   \frac{1}{2} e^{-\alpha_t/2}.
    \end{aligned}
  \end{equation}
  Similarly,
  \begin{equation} \label{eqn:EI-bound-RKHS-pf-2}
    \centering
    \begin{aligned}
          \Pbb\{e_{t-1}(\xbm) < -\alpha_t^{1/2}\sigma_{t-1}^e(\xbm)\} \leq \frac{1}{2} e^{-\alpha_t/2}.
    \end{aligned}
  \end{equation}
  Therefore, 
  \begin{equation} \label{eqn:EI-bound-RKHS-pf-3}
    \centering
    \begin{aligned}
          \Pbb\{|e_{t-1}(\xbm)| > \alpha_t^{1/2}\frac{\sigma_{\epsilon}}{\sigma}\sigma_{t-1}(\xbm)\} \leq  \Pbb\{|e_{t-1}(\xbm)| > \alpha_t^{1/2}\sigma_{t-1}^e(\xbm)\}\leq  e^{-\alpha_t/2}.
    \end{aligned}
  \end{equation}
  Thus, taking the union bound of $\xbm\in C$, 
  \begin{equation} \label{eqn:EI-bound-RKHS-pf-4}
    \centering
    \resizebox{\textwidth}{!}{$
    \begin{aligned}
          \Pbb\{\bigcup_{\xbm\in C} |e_{t-1}(\xbm)| > \alpha_t^{1/2}\frac{\sigma_{\epsilon}}{\sigma}\sigma_{t-1}(\xbm)\}  
          \leq \bigcup_{\xbm\in C} \Pbb\{ |e_{t-1}(\xbm)| > \alpha_t^{1/2}\frac{\sigma_{\epsilon}}{\sigma}\sigma_{t-1}(\xbm)\} 
          \leq |C| e^{-\alpha_t/2}.
    \end{aligned}
$
}
  \end{equation}
  Next, taking the union bound over $t\in\Nbb$, we have   
  \begin{equation} \label{eqn:EI-bound-RKHS-pf-5}
    \centering
    \begin{aligned}
          \Pbb\{\bigcup_{t=1}^{\infty} |e_{t-1}(\xbm)| > \alpha_t^{1/2}\frac{\sigma_{\epsilon}}{\sigma}\sigma_{t-1}(\xbm), \forall \xbm \in C\} 
          \leq |C| \sum_{t=1}^{\infty} e^{-\alpha_t/2}.
    \end{aligned}
  \end{equation}
  Since $\sum_{t=1}^{\infty} \frac{1}{\pi_t}=1$, $|C| \sum_{t=1}^{\infty} e^{-\alpha_t/2} = \delta$. 
  It follows that 
  \begin{equation} \label{eqn:EI-bound-RKHS-pf-6}
  \centering
  \begin{aligned}
        |e_{t-1}(\xbm)|\leq \alpha_t^{1/2} \frac{\sigma_{\epsilon}}{\sigma}\sigma_{t-1}(\xbm),
   \end{aligned}
  \end{equation} 
  with probability $\geq 1-\delta$, for $\forall \xbm\in C$ and $t\in\Nbb$. 
  Further, similar to~\eqref{eqn:rkhs-f-bound-gaussian-pf-4} 
    \begin{equation} \label{eqn:EI-bound-RKHS-pf-7}
  \centering
  \begin{aligned}
        |f(\xbm)-\mu_{t-1}(\xbm)| \leq& |f(\xbm)-\hbm_{t-1}^T(\xbm)\fbm_{1:t-1}|+|e_{t-1}(\xbm)|\\
               \leq& (B+\alpha_t^{1/2}\frac{\sigma_{\epsilon}}{\sigma}) \sigma_{t-1}(\xbm) = \beta_t^{1/2} \sigma_{t-1}(\xbm),
   \end{aligned}
  \end{equation} 
  with probability $\geq 1-\delta$, for $\forall \xbm\in C$ and $t\in\Nbb$. 
\end{proof}
We state an extension of  Theorem~\ref{thm:rkhs-f-bound-finite} using only union bound in $t\in\Nbb$ on $\xbm_t$.  
\begin{lemma}\label{lem:EI-bound-RKHS-t-only}
Given $\delta\in(0,1)$, let  $\beta_t^{1/2} = B+\frac{\sigma_{\epsilon}}{\sigma}\sqrt{2\log(\pi_t /\delta)}$, where $\pi_t$ satisfies $\sum_{t=1}^{\infty} \frac{1}{\pi_t}=1$. Then,   $\forall t\geq 1$,   
\begin{equation} \label{eqn:EI-bound-RKHS-t-only}
 \centering
  \begin{aligned}
     |f(\xbm_{t}) - \mu_{t-1}(\xbm_{t})  | \leq \beta_t^{1/2} \sigma_{t-1}(\xbm_{t}), 
  \end{aligned}
\end{equation}
holds with probability $\geq 1-\delta$.
\end{lemma}

Next, we consider compact and convex $C$ given by Assumption~\ref{assp:constraint}.
In~\cite{srinivas2009gaussian}, the authors assumed a lower bound on the derivatives of the kernel functions $k(\xbm,\xbm')$ with high probability. 
 That is there exist constants $a,b>0$ such that the bounded kernel $k(\xbm,\xbm')\leq 1$ satisfies 
 \begin{equation} \label{eqn:kbounded-1}
  \centering
  \begin{aligned}
     P\{\sup_{x\in C} |\partial f/\partial x_j| > L \} \leq a e^{-(\frac{L}{b})^2},
  \end{aligned}
\end{equation}
While this assumption is only used to derive a Lipschitz continuous $f$ with a high probability,~\eqref{eqn:kbounded-1} implies that $k(\xbm,\xbm')$ is differentiable.  
In~\cite{lederer2019uniform}, the authors assumed Lipschitz kernels defined by their partial derivatives. The author further assumed Lipschitz continuity of $f$.
From a nonsmooth optimization perspective, Lipschitz continuity is a more general condition than differentiability. Therefore, we choose to assume a Lipschitz continuous $f$ directly in Assumption~\ref{assp:lipschitz}.  

We construct a discretization $\Cbb(h)$ of $C$ where each $\Cbb$ comprises finite number of points in $C$. 
The size parameter $h$ is a measure of the distance between $\Cbb$ and $\forall \xbm\in C$.
The discretization aims to cover the compact and convex set $C$ with enough points so that for $\forall x\in C$, the distance between $\xbm\in C$ and  $\Cbb$ can be controlled.  Denote the closest point of $\xbm$ to $\Cbb$ as 
\begin{equation} \label{eqn:close-point}
 \centering
  \begin{aligned}
 	[\xbm] := \underset{\substack{\ubm}\in \Cbb}{\text{minimize}} 
	   \norm{\xbm-\ubm}. \\
  \end{aligned}
\end{equation}
The size parameter $h$ satisfies
\begin{equation} \label{eqn:compactdisc-1}
 \centering
  \begin{aligned}
     \norm{\xbm-[\xbm]} \leq h, \ \forall \xbm\in C. 
  \end{aligned}
\end{equation}
The cardinality $|\Cbb(h)|$ of the set $\Cbb(h)$ is in general reversely correlated with $h$.
Consider sequences $\{h_t\}$ and $\{\Cbb_t(h_t)\}$ where the discretization is refined as $t$ increases. 
We denote the latter as $\Cbb_t$ for simplicity.
We emphasize that the discretization does not play a role in the GP-UCB algorithm, but does for GP-TS. 
An immediate result is a lemma for $\Cbb_t$ similar to that of Lemma~\ref{lem:f-bound-RKHS-1}.
 \begin{lemma}\label{lem:f-discret-bound-RKHS}
   Given $\delta\in(0,1)$, let $\beta_t^{1/2}= B+ c_{\sigma}\sqrt{2\log(|\Cbb_t|\pi_t/\delta)}$ where $\sum_{t=1}^{\infty}\frac{1} {\pi_t}=1$ and $c_{\sigma}=\frac{\sigma_{\epsilon}}{\sigma}$. 
   Then, for $\forall \xbm\in \Cbb_t$ and $\forall t\in\Nbb$,
      \begin{equation} \label{eqn:f-discret-bound-RKHS-1}
  \centering
  \begin{aligned}
        |f(\xbm)-\mu_{t-1}(\xbm)|\leq \beta_t^{1/2} \sigma_{t-1}(\xbm),
   \end{aligned}
  \end{equation} 
  with probability $\geq 1-\delta$.
\end{lemma}

Let $N_{h_t}$ be the minimum number of points needed to achieve~\eqref{eqn:compactdisc-1}. 
Given $C\subseteq [0,r]^d$, we can construct the discretization with $N_{h_t}$, so that 
   \begin{equation} \label{eqn:disc-size-1}
  \centering
  \begin{aligned}
     |\Cbb_t| =   N_{h_t} = (1+ \frac{r}{h_t})^d,
   \end{aligned}
  \end{equation}
when the points in $\Cbb_t$ are evenly spaced. 
In order to achieve no regret, we require $\sum_{t=1}^{\infty}h_t < \infty$. In this section, we adopt the common choice of $h_t$ for simplicity of presentation:
   \begin{equation} \label{eqn:ht-1}
  \centering
  \begin{aligned}
    h_t = \frac{1}{L t^{2}}, 
   \end{aligned}
  \end{equation}
  where the Lipschitz constant $L$ is defined with respect to the 2-norm.
  The corresponding cardinality of $\Cbb_t$ is $|\Cbb_t| = (1+Lrt^2)^d$.

\begin{lemma}\label{lem:compact-f-RKHS}
Given $\delta\in(0,1)$, let  $\beta_t^{1/2} =  B+c_{\sigma}\sqrt{2\log(\pi_t /\delta) + 2d \log(1+rt^2L)} $, where $\sum_{t=1}^{\infty}\frac{1}{\pi_t}=1$ and $c_{\sigma}=\frac{\sigma_{\epsilon}}{\sigma}$. 
Under Assumption~\ref{assp:lipschitz} and~\ref{assp:constraint}, with probability $\geq 1-\delta$, 
   \begin{equation} \label{eqn:compact-f-RKHS-1}
  \centering
  \begin{aligned}
          |f(\xbm^*) - \mu_{t-1}([\xbm^*]_t)| \leq \beta_t^{1/2} \sigma_{t-1}([\xbm^*]_t) + \frac{1}{t^2},
  \end{aligned}
  \end{equation}
  for $\forall t\in\Nbb$.
\end{lemma}
\begin{proof}
  From Assumption~\ref{assp:constraint}, we know 
  \begin{equation} \label{eqn:compact-f-RKHS-pf-1}
  \centering
  \begin{aligned}
     | f(\xbm) - f(\xbm') | < L \norm{\xbm-\xbm'}.
  \end{aligned}
\end{equation}
   Applying the discretization $\Cbb_t$,~\eqref{eqn:ht-1}, and~\eqref{eqn:compactdisc-1} to~\eqref{eqn:compact-f-RKHS-pf-1}, for $\forall \xbm\in C$, we have 
  \begin{equation} \label{eqn:compact-f-RKHS-pf-2}
  \centering
  \begin{aligned}
     | f(\xbm_t) - f([\xbm_t]) | < L h_t=\frac{1}{t^2}.
  \end{aligned}
\end{equation}
   By~\eqref{eqn:disc-size-1}, $|\Cbb_t| = (1+ r t^2 L)^d$. 
   Thus, the choice of $\beta_t$ in the condition of the lemma satisfies Lemma~\ref{lem:f-discret-bound-RKHS}. 
   Using~\eqref{eqn:compact-f-RKHS-pf-2} and Lemma~\ref{lem:f-discret-bound-RKHS}, 
  \begin{equation} \label{eqn:compact-f-RKHS-pf-3}
  \centering
  \begin{aligned}
          |f(\xbm^*) - \mu_{t-1}([\xbm^*]_t)| \leq& |f(\xbm^*) - f([\xbm^*]_t)| + |f([\xbm^*]_t) - \mu_{t-1}([\xbm^*]_t)|\\ 
        \leq& \beta_{t}^{1/2} \sigma_{t-1}([\xbm^*]_t) + \frac{1}{t^2},
  \end{aligned}
\end{equation}
  with probability $1-\delta$.
\end{proof}

\section{Improved cumulative regret bound for GP-UCB}\label{se:regret}
In this section, we consider the cumulative regret of GP-UCB under the Assumptions~\ref{assp:rkhs},~\ref{assp:lipschitz},~\ref{assp:constraint}, and~\ref{assp:gaussiannoise}.
The instantaneous regret $r_t$ is defined as 
 \begin{equation} \label{eqn:rt-1}
  \centering
  \begin{aligned}
  r_t = f(\xbm^*)- f(\xbm_{t}).
  \end{aligned}
\end{equation}
For simplicity of presentation, we choose $\pi_t=\frac{\pi^2t^2}{6}$.
The following two Lemmas are well-established results from~\cite{srinivas2009gaussian} on the information gain and variances.
\begin{lemma}\label{lem:infogain}
   Given observations  $\ybm_T=[y_1,\dots,y_T]$ and function values $\fbm_T=[f(\xbm_1),\dots,f(\xbm_T)]$ at the set of points $A_t=\{\xbm_1,\dots,\xbm_T\}$, the information gain is 
    \begin{equation} \label{eqn:info-1}
  \centering
  \begin{aligned}
       I(\ybm_T;\fbm_T) = \frac{1}{2} \sum_{t=1}^T \log(1+\sigma^{-2}\sigma_{t-1}^2(\xbm_t)). 
  \end{aligned}
\end{equation}
\end{lemma}
As in~\cite{srinivas2009gaussian,bull2011convergence},  we use  $\gamma_T$ to bound the sum of variance at $A_t$.
\begin{lemma}\label{lem:variancebound}
 The sum of posterior variance functions $\sigma_{t}(\xbm)$ satisfies
 \begin{equation} \label{eqn:var-1}
  \centering
  \begin{aligned}
    \sum_{t=1}^T  \sigma_{t-1}^2(\xbm) \leq \frac{1}{4}C_{\gamma} \gamma_T, 
  \end{aligned}
\end{equation}
 where $C_{\gamma} = 8/log(1+\sigma^{-2})$.
\end{lemma}
The instantaneous regret is bounded in the following lemma.
\begin{lemma}\label{lem:compact-rkhs-instregret}
Given $\delta\in(0,1)$, let  $\beta_t^{1/2} =  B+c_{\sigma}\sqrt{2\log(2\pi_t /\delta) + 2d \log(1+rt^2L)} $, 
where $\sum_{t=1}^{\infty}\frac{1}{\pi_t}=1$ and $c_{\sigma}=\frac{\sigma_{\epsilon}}{\sigma}$. 
 Then, with probability $\geq 1-\delta$, 
   \begin{equation} \label{eqn:compact-rhks-instregret-1}
  \centering
  \begin{aligned}
          r_t \leq&    2\beta_t^{1/2}\sigma_{t-1}(\xbm_t)+\frac{1}{t^2}
  \end{aligned}
  \end{equation}
\end{lemma}
\begin{proof}
   Notice that the $\beta_t$ in this Lemma satisfies the conditions in both Lemma~\ref{lem:EI-bound-RKHS-t-only} and~\ref{lem:compact-f-RKHS} with $\delta/2$. Thus, with probability $\geq 1-\delta$,  
   \begin{equation} \label{eqn:compact-rhks-instregret-pf-1}
  \centering
  \begin{aligned}
          r_t  =& f(\xbm^*)- f(\xbm_{t})   
          \leq  \mu_{t-1}([\xbm^*]_t) + \beta_t^{1/2} \sigma_{t-1}([\xbm^*]_t)+\frac{1}{t^2}-f(\xbm_t)\\
          \leq&  \mu_{t-1}(\xbm_t)-f(\xbm_t) + \beta_t^{1/2} \sigma_{t-1}(\xbm_t)+\frac{1}{t^2}\\
          \leq&  2\beta_t^{1/2}\sigma_{t-1}(\xbm_t)+\frac{1}{t^2},
  \end{aligned}
  \end{equation}
  where the first line uses Lemma~\ref{lem:compact-f-RKHS}, the second line uses the definition of $\xbm_t$, and the last line uses Lemma~\ref{lem:EI-bound-RKHS-t-only}.
\end{proof}
\begin{theorem}\label{thm:compact-regretbound-rkhs}{\small 
Given $\delta\in(0,1)$, let  $\beta_t^{1/2} = B+ c_{\sigma}\sqrt{2\log(2 \pi_t/\delta) + 2d \log( 1+rt^2L)} $, 
where $\sum_{t=1}^{\infty}\frac{1}{\pi_t}=1$ and $c_{\sigma}=\frac{\sigma_{\epsilon}}{\sigma}$. 
 Then, with probability $\geq 1-\delta$, the GP-UCB Algorithm~\ref{alg:boucb} maintains the regret bound}
 \begin{equation}\label{eqn:compact-regretbound-rkhs-1}
  \centering
  \begin{aligned}
     P\left\{ R_T \leq (C_1   T \beta_T  \gamma_T)^{1/2}+2 \right\}  \geq 1-\delta,
  \end{aligned}
\end{equation}
for any $T$, where $C_1=8/log(1+\sigma^{-2})$. 
\end{theorem}
\begin{proof}
   From Lemma~\ref{lem:compact-rkhs-instregret}, with probability $\geq 1-\delta$, 
 \begin{equation}\label{eqn:compact-regretbound-rkhs-pf-1}
  \centering
  \begin{aligned}
      R_T =\sum_{t=1}^T r_t \leq \sum_{t=1}^T 2 \beta_t^{1/2}\sigma_{t-1}(\xbm_t) +\frac{1}{t^2} \leq 2 \beta_T^{1/2} \sum_{t=1}^T \sigma_{t-1}(\xbm_t) +\sum_{t=1}^T \frac{1}{t^2}. 
   \end{aligned}
\end{equation}
   Using Lemma~\ref{lem:infogain},~\ref{lem:variancebound}, and the Cauchy Schwarz inequality, we know $\sum_{t=1}^T \sigma_{t-1}(\xbm_t) \leq \frac{1}{2} (C_1 T\gamma_T)^{1/2}$.
   Therefore,~\eqref{eqn:compact-regretbound-rkhs-pf-1} simplifies to  
 \begin{equation}\label{eqn:compact-regretbound-rkhs-pf-2}
  \centering
  \begin{aligned}
      R_T \leq  2 \beta_T^{1/2} \frac{1}{2}(C_1 T\gamma_T)^{1/2} +\sum_{t=1}^T \frac{1}{t^2} = (C_1  T \beta_T\gamma_T)^{1/2} +\sum_{t=1}^T \frac{1}{t^2}. 
   \end{aligned}
\end{equation}

\end{proof}

\begin{remark}\label{remark:betterreg}
Let $\pi_t = \frac{\pi^2t^2}{6}$.
The regret bound is of $\mathcal{O}(\sqrt{T\gamma_T} (B+\sqrt{d\log(T)}))$. 
Under the frequentist setting with bounded noise,~\cite{srinivas2009gaussian} proves a regret bound of UCB: $\mathcal{O}(\sqrt{T\gamma_T}(B+\sqrt{\gamma_T}\log^{\frac{3}{2}}(T)))$. 
In~\cite{chowdhury2017kernelized}, under the RKHS assumption and sub-Gaussian noise, the UCB regret bound is of $\mathcal{O}(\sqrt{T\gamma_T}(B+\sqrt{\gamma_T}))$.
Clearly, the Gaussian noise assumption in this paper produces an improved regret bound as long as $d\log(T) = o(\gamma_T)$, which is the case for SE and Matérn kernels for a finite $d$. Specifically, for SE kernel, $\gamma_T=\mathcal{O}(\log^{d+1}(T))$ and for Matérn kernel $\gamma_T=\mathcal{O}(T^{\frac{d}{2\nu+d}}\log^{\frac{2\nu}{2\nu+d}}(T))$, where $\nu$ is the parameter of the kernel,~\cite{srinivas2009gaussian,vakili2021information}. Indeed, the regret bound is the same in order as that in~\cite{srinivas2009gaussian} but under the Bayesian setting.
Recently, in~\cite{wang2023regret}, again under the RKHS assumption and sub-Gaussian noise, the authors claim the state-of-the-art UCB optimal regret bound for SE kernel $\mathcal{O}(\sqrt{T}\log^{\frac{d+3}{2}}(T))$.
In comparsion, our regret bound for SE kernel is $\mathcal{O}(\sqrt{T}\sqrt{d}\log^{\frac{d+2}{2}}(T))$.
That is, for finite dimension and Gaussian noise, our rate remains an improvement, and is closer to the lower bound established in~\cite{scarlett2017lower}.
\end{remark}

\subsection{Improved regret bounds for Thompson Sampling}\label{se:thompson}
In this section, we present the improved cumulative regret for Thompson sampling on two commonly used kernels. Our analysis follows that in~\cite{chowdhury2017kernelized} under Assumptions~\ref{assp:rkhs},~\ref{assp:lipschitz},~\ref{assp:constraint}, and~\ref{assp:gaussiannoise}.
We note the choice of $\nu_t$ in the GP-TS algorithm is closely related to the choices of $\beta_t$, while the discretization $\Cbb_t$ in Algorithm~\ref{alg:bots} is the same as that used in section~\ref{se:regret}, \textit{i.e.}, $|\Cbb_t|=(1+rt^2L)^d$ is in~\eqref{eqn:disc-size-1}.
Let 
 \begin{equation} \label{def:ct}
  \centering
  \begin{aligned}
   c_t =\sqrt{2\log(|\Cbb_t|t^2)}=\sqrt{2d\log((1+rLt^2)t^2)}, \ \zeta_t^{1/2} = \nu_t^{1/2}(1+c_t). 
   \end{aligned}
  \end{equation} 
The analysis in this section is performed under these choices of algorithmic parameters.
We define some technical quantities for ease of presentation.
\begin{definition}\label{def:saturated-set}
The set of saturated points $\Sbb_t\subset \Cbb_t$ is defined as  
   \begin{equation} \label{eqn:saturated-set}
  \centering
  \begin{aligned}
         \Sbb_t:= \{\xbm \in \Cbb_t: \Delta_t(\xbm) > \zeta_t^{1/2} \sigma_{t-1}(\xbm)\}.
   \end{aligned}
  \end{equation} 
  where $\Delta_t(\xbm):=f([\xbm^*]_t)- f(\xbm)$.
\end{definition}

\begin{definition}\label{def:eft}
  The event $E^f(t)$ is defined as follows: for $\forall \xbm \in \Cbb_t$,  
   \begin{equation} \label{eqn:eft-1}
  \centering
  \begin{aligned}
        |f(\xbm)-\mu_{t-1}(\xbm)|\leq \nu_t^{1/2} \sigma_{t-1}(\xbm),
   \end{aligned}
  \end{equation} 
  The event $E^{f_t}(t)$ is defined by: for $\forall \xbm\in \Cbb_t$,
   \begin{equation} \label{eqn:eft-2}
  \centering
  \begin{aligned}
        |f_t(\xbm)-\mu_{t-1}(\xbm)|\leq \nu_t^{1/2}c_t \sigma_{t-1}(\xbm),
   \end{aligned}
  \end{equation} 
\end{definition}
The filtration $\mathcal{F}_{t-1}'$  is defined as the history until $t$. The set $\Sbb_t$, $\mu_{t-1}(\xbm)$ and $\sigma_{t-1}(\xbm)$ are deterministic given $\mathcal{F}'_{t-1}$.
For simplicity, if needed, we write the conditional quantities given $\mathcal{F}_{t-1}'$ with subsript $t$, \textit{e.g.}, $\Ebb[\cdot|\mathcal{F}'_{t-1}]=\Ebb_t$.

\noindent We are now ready to present the analysis. First, the prediction error bound with $\nu_t$ is given.
 \begin{lemma}\label{lem:f-discrete-bound-RKHS-1}
   Given $\delta\in(0,1)$ and $\nu_t$~\eqref{def:nut},
      \begin{equation} \label{eqn:f-discrete-bound-RKHS-1}
  \centering
  \begin{aligned}
        |f(\xbm)-\mu_{t-1}(\xbm)|\leq \nu_t^{1/2} \sigma_{t-1}(\xbm),
   \end{aligned}
  \end{equation} 
  with probability $\geq 1-\delta/2$, for $\forall \xbm\in \Cbb_t$ and $t\in\Nbb$.
\end{lemma}
Lemma~\ref{lem:f-discrete-bound-RKHS-1} is a direct result of Lemma~\ref{lem:f-discret-bound-RKHS} with $\delta/2$. 

\begin{lemma}\label{lem:ts-1}
   Given $\mathcal{F}'_{t-1}$, $\Cbb_t$, $c_t$~\eqref{def:ct}, and $\nu_t$~\eqref{def:nut}, for $\forall \xbm\in \Cbb_t$,
  \begin{equation} \label{eqn:ft-bound-RKHS-1}
  \centering
  \begin{aligned}
        |f_t(\xbm)-\mu_{t-1}(\xbm)|\leq \nu_t^{1/2} c_t \sigma_{t-1}(\xbm),
   \end{aligned}
  \end{equation} 
  with probability $\geq 1-\frac{1}{t^2}$. 
\end{lemma}
\begin{proof}
   Since $f_t(\xbm)\sim\mathcal{N}(\mu_{t-1}(\xbm), \nu_t \sigma_{t-1}^2(\xbm))$, from Lemma~\ref{lem:phi}, we have 
       \begin{equation} \label{eqn:ft-bound-RKHS-pf-1}
  \centering
  \begin{aligned}
        \Pbb\{|f_t(\xbm)-\mu_{t-1}(\xbm)| > \nu_t^{1/2} \sqrt{2\log(|\Cbb_t|t^2)} \sigma_{t-1}(\xbm)\}\leq& e^{-\frac{1}{2} (2\log(|\Cbb_t| t^2)}
        \leq \frac{1}{|\Cbb_t|t^2}.
   \end{aligned}
  \end{equation} 
  Therefore, using union bound on $\Cbb_t$,~\eqref{eqn:ft-bound-RKHS-1} is true with probability $\geq 1-\frac{1}{t^2}$.
\end{proof}
From Lemma~\ref{lem:f-discrete-bound-RKHS-1} and~\ref{lem:ts-1}, we can directly infer the probability of $E^f(t)$ and $E^{f_t}(t)$.
\begin{lemma}\label{lem:eftprob}
     Given $\delta\in(0,1)$ and the parameters $\nu_t$, $c_t$, and $\zeta_t$, $E^f(t), \forall t\in\Nbb$ is true with probability $\geq 1-\delta/2$ while $E^{f_t}(t)$ is true with probability $\geq 1-\frac{1}{t^2}$ at $t$.
\end{lemma}
The following Lemma gives a well-knonn lower bound a probability of the Gaussian distribution.
\begin{lemma}\label{lem:gaussianprob}
    Given a random variable $X\sim\mathcal{N}(\mu,\sigma)$, for $\forall w>0$,
  $\Pbb\{\frac{X-\mu}{\sigma}>w\} \geq \frac{e^{-w^2}}{4\sqrt{\pi}w}$.
\end{lemma}
\begin{lemma}\label{lem:ts-prob-bound-1}
    For any filtration $\mathcal{F}'_{t-1}$, if $E^f(t)$ is true,  
   \begin{equation} \label{eqn:ts-prob-bound-1}
  \centering
  \begin{aligned}
      \Pbb_t\{ f_t(\xbm)>f(\xbm)\}\geq p,
   \end{aligned}
  \end{equation} 
   for $\forall \xbm\in\Cbb_t$ and $p=\frac{1}{4e\sqrt{\pi}}$.
\end{lemma}
\begin{proof}
    Recall that $f_t(\xbm)\sim\mathcal{N}(\mu_{t-1}(\xbm),\nu_t^{1/2}\sigma_{t-1}(\xbm))$ given $\mathcal{F}'_{t-1}$.
    For $\forall \xbm\in\Cbb_t$, by $E^f(t)$, 
   \begin{equation} \label{eqn:ts-prob-bound-pf-1}
  \centering
  \begin{aligned}
      \Pbb_t\{ f_t(\xbm)>f(\xbm)\} =& \Pbb_t\left\{\frac{f_t(\xbm)-\mu_{t-1}(\xbm)}{\nu_t\sigma_{t-1}(\xbm)} >  \frac{f(\xbm)-\mu_{t-1}(\xbm)}{\nu_t\sigma_{t-1}(\xbm)}  \right\} \\
       \geq& \Pbb_t\left\{\frac{f_t(\xbm)-\mu_{t-1}(\xbm)}{\nu_t\sigma_{t-1}(\xbm)} >  \frac{|f(\xbm)-\mu_{t-1}(\xbm)|}{\nu_t\sigma_{t-1}(\xbm)}  \right\}\\
       \geq& \Pbb_t\left\{\frac{f_t(\xbm)-\mu_{t-1}(\xbm)}{\nu_t\sigma_{t-1}(\xbm)} > 1 \right\} \geq \frac{1}{4\sqrt{\pi}e},\\
   \end{aligned}
  \end{equation} 
  where the second inequality uses the fact $E^f(t)$ is true and the last inequality is from Lemma~\ref{lem:gaussianprob}.
\end{proof}

\begin{lemma}\label{lem:ts-prob-bound-2}
    For any filtration $\mathcal{F}'_{t-1}$, if $E^f(t)$ is true,
   \begin{equation} \label{eqn:ts-prob-bound-2}
  \centering
  \begin{aligned}
      \Pbb_t\{ \xbm_t\in \Cbb_t \backslash \Sbb_t \}\geq p - \frac{1}{t^2}.
   \end{aligned}
  \end{equation} 
\end{lemma}
\begin{proof}
   Recall that $f_t(\xbm_t) \geq f_t(\xbm)$, for $\forall \xbm\in\Cbb_t$.
   Suppose for $[\xbm^*]_t\in\Cbb_t$, $f_t([\xbm^*]_t) > f_t(\xbm)$, for $\forall \xbm\in \Sbb_t$.  Then, $f_t(\xbm_t)\geq f_t([\xbm^*]_t) > f_t(\xbm)$ and $\xbm_t\in \Cbb_t \backslash \Sbb_t$. This implies
    \begin{equation} \label{eqn:ts-prob-bound-2-pf-1}
  \centering
  \begin{aligned}
      \Pbb_t\{ \xbm_t\in \Cbb_t \backslash \Sbb_t \}\geq \Pbb_t\{f_t([\xbm^*]_t) >f_t(\xbm), \forall x\in\Sbb_t  \}.
   \end{aligned}
  \end{equation} 
   By definition~\eqref{eqn:saturated-set}, $\Delta_t(\xbm) > \zeta_t^{1/2}\sigma_{t-1}(\xbm)$, for $\forall \xbm\in\Sbb_t$.
   Suppose $E^f(t)$ and $E^{f_t}(t)$ are both true, then 
     \begin{equation} \label{eqn:ts-prob-bound-2-pf-2}
  \centering
  \begin{aligned}
         f_t(\xbm) -f(\xbm) \leq \mu_{t-1}(\xbm) + c_t\nu^{1/2}_t\sigma_{t-1}(\xbm) -\mu_{t-1}(\xbm)+\nu^{1/2}_t\sigma_{t-1}(\xbm) = \zeta_t^{1/2} \sigma_{t-1}(\xbm), 
   \end{aligned}
  \end{equation} 
  for $\forall \xbm\in \Cbb_t$. Thus, $\forall \xbm\in \Sbb_t$, 
     \begin{equation} \label{eqn:ts-prob-bound-2-pf-3}
  \centering
  \begin{aligned}
       f_t(\xbm) < f(\xbm) + \Delta_t(\xbm) = f([\xbm]^*_t).
   \end{aligned}
  \end{equation}
   Therefore, if $E^f(t)$ holds, either $E^{f_t}(t)$ is false or~\eqref{eqn:ts-prob-bound-2-pf-3} is true. 
   Next, notice that given~\eqref{eqn:ts-prob-bound-2-pf-3} true, 
   $f_t([\xbm^*]_t)> f([\xbm^*]_t)$ guarantees $f_t([\xbm^*]_t)>f_t(\xbm)$ for $\forall \xbm\in \Sbb_t$.
   Hence, given $E^f(t)$ and~\eqref{eqn:ts-prob-bound-2-pf-3},
    \begin{equation} \label{eqn:ts-prob-bound-2-pf-4}
  \centering
  \begin{aligned}
   \Pbb_t\{f([\xbm^*]_t) >f_t(\xbm) , \forall x\in\Sbb_t  \}\geq \Pbb_t\{  f_t([\xbm^*]_t)> f([\xbm^*]_t)\}. 
   \end{aligned}
  \end{equation}
   Further, given $E^f(t)$, 
    \begin{equation} \label{eqn:ts-prob-bound-2-pf-5}
  \centering
  \begin{aligned}
   \Pbb_t\{f([\xbm^*]_t) >f_t(\xbm) , \forall x\in\Sbb_t  \}+ \Pbb_t\{\overline{E^{f_t}(t)} \}\geq \Pbb_t\{  f_t([\xbm^*]_t)> f([\xbm^*]_t)\}, 
   \end{aligned}
  \end{equation}
   where $\overline{E^{f_t}(t)}$ is the complement of $E^{f_t}(t)$.
   Using Lemma~\ref{lem:ts-1} and Lemma~\ref{lem:ts-prob-bound-1},~\eqref{eqn:ts-prob-bound-2-pf-5} becomes
     \begin{equation} \label{eqn:ts-prob-bound-2-pf-6}
  \centering
  \begin{aligned}
   \Pbb_t\{f_t([\xbm^*]_t) >f_t(\xbm), \forall x\in\Sbb_t  \} \geq  \Pbb_t\{ f_t([\xbm^*]_t)> f([\xbm^*]_t)\}- \frac{1}{t^2} \geq p -\frac{1}{t^2},
   \end{aligned}
  \end{equation}
   Combine~\eqref{eqn:ts-prob-bound-2-pf-1} with~\eqref{eqn:ts-prob-bound-2-pf-6}, the proof is complete.
\end{proof}

\begin{lemma}\label{lem:ts-r-exp}
   For any filtration $\mathcal{F}_{t-1}'$, if $E^f(t)$ is true,
   \begin{equation} \label{eqn:ts-r-exp-1}
  \centering
  \begin{aligned}
      \Ebb_t[r_t] \leq \frac{11\zeta_t^{1/2}}{p} \Ebb_t[\sigma_{t-1}(\xbm_t)] +\frac{2B+1}{t^2}
   \end{aligned}
  \end{equation} 
\end{lemma}

\begin{proof}
   First, we define $\bar{\xbm}_t$ as the unsaturated point with the smallest $\sigma_{t-1}(\xbm)$, \textit{i.e.},
   \begin{equation} \label{eqn:ts-r-exp-pf-1}
  \centering
  \begin{aligned}
                \bar{\xbm}_t = &\underset{\substack{\xbm}\in \Cbb_t\backslash\Sbb_t}{\text{argmin}} 
	    \sigma_{t-1}(\xbm). \\
   \end{aligned}
  \end{equation} 
   Notice that $\bar{\xbm}_t$ is deterministic given $\mathcal{F}_{t-1}'$. 
   Since $E^f(t)$ is true, 
   \begin{equation} \label{eqn:ts-r-exp-pf-2}
  \centering
  \begin{aligned}
           \Ebb_t[\sigma_{t-1}(\xbm_t)] \geq \Ebb_t[\sigma_{t-1}(\xbm_t),\xbm_t\in \Cbb_t\backslash\Sbb_t] \Pbb[\xbm_t\in \Cbb_t\backslash\Sbb_t] \geq \sigma_{t-1}(\bar{\xbm}_t)(p-\frac{1}{t^2}),
   \end{aligned}
  \end{equation} 
  where the last inequality is from Lemma~\ref{lem:ts-prob-bound-2}.
  Suppose now $E^{f_t}(t)$ is true. We know
  $-\zeta_t\sigma_{t-1}(\xbm)\leq f(\xbm)-f_t(\xbm)\leq \zeta_t\sigma_{t-1}(\xbm)$, for $\forall \xbm\in\Cbb_t$. Further, recall by definition $f_t(\xbm_t)\geq f_t(\xbm)$ for $\forall \xbm\in\Cbb_t$. Thus,  
    \begin{equation} \label{eqn:ts-r-exp-pf-3}
  \centering
  \begin{aligned}
     \Delta_t(\xbm_t)= f([\xbm^*]_t) - f(\xbm_t)  =& f([\xbm^*]_t) -f(\bar{\xbm}_t) + f(\bar{\xbm}_t) - f(\xbm_t)\\
                      \leq& \Delta_t(\bar{\xbm}_t) - f_t(\xbm_t)+\zeta_t^{1/2}\sigma_{t-1}(\xbm_t) +f_t(\bar{\xbm}_t)+\zeta_t^{1/2}\sigma_{t-1}(\bar{\xbm}_t) \\
                   \leq&  \zeta_t^{1/2}\sigma_{t-1}(\xbm_t) + 2\zeta_t^{1/2}\sigma_{t-1}(\bar{\xbm}_t),
   \end{aligned}
  \end{equation} 
  where the third line uses $\bar{\xbm}_t\in \Cbb_t\backslash \Sbb_t$ and $f_t(\xbm_t)\geq f_t(\xbm), \forall\xbm\in \Cbb_t$.
  Thus, given $E^f(t)$, either~\eqref{eqn:ts-r-exp-pf-3} is true or $E^{f_t}(t)$ is false.  
  Since $f$ is bounded, $\Delta_t(\xbm)\leq 2 \sup_{\xbm\in C} |f(\xbm)|<\leq 2B$. Taking the expectation of~\eqref{eqn:ts-r-exp-pf-3} and considering $\bar{E^{f_t}(t)}$, we have  
    \begin{equation} \label{eqn:ts-r-exp-pf-4}
  \centering
  \begin{aligned}
     \Ebb_t\left[\Delta_t(\xbm_t)\right] 
                   \leq&  \Ebb_t\left[\zeta_t^{1/2}(\sigma_{t-1}(\xbm_t) + 2\sigma_{t-1}(\bar{\xbm}_t))\right]+2B\Pbb\left[ \overline{E^{f_t}(t)}\right]\\
           \leq& \zeta_t^{1/2}(\frac{2}{p-\frac{1}{t^2}}+1) \Ebb_t[\sigma_{t-1}(\xbm_t)] + \frac{B}{t^2}\\
           \leq& \frac{11}{p}\zeta_t^{1/2}\Ebb_t[\sigma_{t-1}(\xbm_t)] + \frac{B}{t^2},
   \end{aligned}
  \end{equation} 
  where the second inequality uses~\eqref{eqn:ts-r-exp-pf-2} and Lemma~\ref{eqn:ft-bound-RKHS-1}. The last inequality uses $\frac{1}{p-\frac{1}{t^2}}\leq \frac{5}{p}$. The instantaneous regret can be written as 
    \begin{equation} \label{eqn:ts-r-exp-pf-5}
  \centering
  \begin{aligned}
         r_t =& f(\xbm^*)-f(\xbm_t) =  f(\xbm^*) -  f([\xbm^*]_t)+f([\xbm^*]_t) -f(\xbm_t)
         \leq \Delta_t(\xbm_t) +\frac{1}{t^2}.
    \end{aligned}
  \end{equation} 
  Taking the conditional expectation of~\eqref{eqn:ts-r-exp-pf-5} gives us~\eqref{eqn:ts-r-exp-1}. 
\end{proof}
We now define several quantities based on probability of $E^f(t)$. 
    \begin{equation} \label{eqn:ts-def-5}
  \centering
  \begin{aligned}
   \bar{r}_t =& r_t \cdot I\{E^f(t)\},\\
  X_t =& \bar{r}_t -\frac{11\zeta_t^{1/2}}{p}\sigma_{t-1}(\xbm_t) -\frac{2B+1}{t^2},\\
     Y_t =& \sum_{i=1}^t X_i,
    \end{aligned}
  \end{equation} 
where $Y_0=0$ and $I$ is an indication function that takes value of $1$ if $E^f(t)$ is true and $0$ otherwise. The super-martingale series is defined next.
\begin{definition}\label{def:supmart}
   A sequence of random variables $\{A_t\},t=0,1,\dots$ is called a super-martingale of $\mathcal{F}_t$ is for all $t$, $A_t$ is $\mathcal{F}_t$-measurable, and for $t\geq 1$,
    \begin{equation} \label{eqn:supmart-1}
  \centering
  \begin{aligned}
     \Ebb_t[Z_t] \leq Z_{t-1}.
    \end{aligned}
  \end{equation} 
\end{definition}
\begin{lemma}\label{lem:supmart}
    If a super-martingale $\{A_t\}$ of $\mathcal{F}_t$ satisfies $|A_t-A_{t-1}|\leq \eta_t$ for some constant $\eta_t$, $t=1,\dots,T$, then for $\forall \delta\geq 0$,
    \begin{equation} \label{eqn:supmart-2}
  \centering
  \begin{aligned}
       \Pbb\left[A_T-A_{0}\leq \sqrt{2\log(1/\delta)\sum_{t=1}^T \eta_t^2}\right]\geq 1-\delta.
    \end{aligned}
  \end{equation} 
\end{lemma}
\begin{lemma}
     The sequence $Y_t$, $t=0,\dots,T$ is a super-martingale with respect to $\mathcal{F}_t'$.
\end{lemma}
\begin{proof}
   From the definition of $Y_t$, 
    \begin{equation} \label{eqn:supmart-2-pf-1}
  \centering
  \begin{aligned}
   \Ebb_t[Y_t-Y_{t-1}] = \Ebb_t[X_t] =& \Ebb_t\left[\bar{r}_t-\frac{11\zeta_t^{1/2}}{p}\sigma_{t-1}(\xbm_t)-\frac{2B+1}{t^2}\right] \\
   \end{aligned}
  \end{equation}
  If $E^{f}(t)$ is false, $\bar{r}_t = 0$ and~\eqref{eqn:supmart-2-pf-1} $\leq 0$.
  If $E^{f}(t)$ is true, by Lemma~\ref{lem:ts-r-exp},~\eqref{eqn:supmart-2-pf-1} $\leq 0$.
\end{proof}

\begin{lemma}\label{lemma:ts-regret-1}
    Given $\delta\in(0,1)$, with probability $\geq 1-\delta$,
     \begin{equation} \label{eqn:ts-rt-1}
  \centering
  \begin{aligned}
      R_T\leq \frac{11\zeta_T^{1/2}}{p} \sum_{t=1}^T \sigma_{t-1}(\xbm_t) + \frac{(2B+1)\pi^2}{6}+\frac{(4B+11)\zeta_T}{p}\sqrt{2T\log(2/\delta)}.
    \end{aligned}
  \end{equation} 
\end{lemma}
\begin{proof}
      From the definition of $Y_t$, we have
      \begin{equation} \label{eqn:ts-rt-pf-1}
  \centering
  \begin{aligned}
           |Y_t-Y_{t-1}| = |X_t| \leq |\bar{r}_t| + \frac{11\zeta_t^{1/2}}{p}\sigma_{t-1}(\xbm_t) +\frac{2B+1}{t^2}.
    \end{aligned}
  \end{equation} 
  Using the upper bound $\bar{r}_t\leq r_t \leq 2B$ and $\sigma_{t-1}(\xbm)\leq 1$, we can write
       \begin{equation} \label{eqn:ts-rt-pf-2}
  \centering
  \begin{aligned}
           |Y_t-Y_{t-1}| \leq 2B+ \frac{11\zeta_t^{1/2}}{p}\sigma_{t-1}(\xbm_t) +\frac{2B+1}{t^2}\leq \frac{(4B+11)\zeta_t^{1/2}}{p}.
    \end{aligned}
  \end{equation} 
   Applying Lemma~\ref{lem:supmart} with $\delta/2$, we have with probability $\geq 1-\delta/2$,
        \begin{equation} \label{eqn:ts-rt-pf-3}
  \centering
  \begin{aligned}
     \sum_{t=1}^T\bar{r}_t \leq& \sum_{t=1}^T\frac{11\zeta_t^{1/2}}{p} \sigma_{t-1}(\xbm_t) +\sum_{t=1}^T \frac{2B+1}{t^2}  +\sqrt{2\log(2/\delta) \sum_{t=1}^T \frac{(4B+11)^2\zeta_t}{p^2}}\\
         \leq&  \frac{11\zeta_T^{1/2}}{p}\sum_{t=1}^T \sigma_{t-1}(\xbm_t) + \frac{(2B+1)\pi^2}{6}  + \frac{(4B+11)\zeta_T^{1/2}}{p} \sqrt{2\log(2/\delta)}.
    \end{aligned}
  \end{equation} 
   Since $R_T = \sum_{t=1}^T r_t$ and $r_t=\bar{r}_t$ with probability $\geq 1-\delta/2$, we obtain~\eqref{eqn:ts-rt-1}. 
\end{proof} 

\begin{theorem}
    Given $\delta\in (0,1)$, the GP-TS algorithm leads to the regret bound  
    \begin{equation} \label{eqn:ts-regret-1}
  \centering
  \begin{aligned}
       R_T = \mathcal{O}(d\log(T)\sqrt{T} \sqrt{\gamma_T}),
    \end{aligned}
  \end{equation}
  with probability $\geq 1-\delta$ for SE and Matérn kernels. Specifically, for SE kernel, $R_T=\mathcal{O}(\sqrt{T}\log^{\frac{d+3}{2}}(T))$ and for Matérn kernels $R_T=\mathcal{O}(T^{\frac{\nu+d}{2\nu+d}}\log^{\frac{3\nu+d}{2\nu+d}}(T) )$.
\end{theorem}
\begin{proof}
    From Lemma~\ref{lem:infogain},~\ref{lem:variancebound}, we know $\sum_{t=1}^T\sigma_{t-1}(\xbm_t) = \mathcal{O}(\sqrt{T\gamma_T})$. We can write $\zeta_t^{1/2}$ as
    \begin{equation} \label{eqn:ts-regret-pf-1}
  \centering
  \begin{aligned}
       \zeta_t^{1/2} = ( B+2\sqrt{\log(3t^2/\delta)+2d\log(1+rLt^2) })(1+\sqrt{2d\log((1+rLt^2))+4\log(t)})
    \end{aligned}
  \end{equation} 
   Therefore, the rate of $\zeta_t^{1/2}$ is 
    \begin{equation} \label{eqn:ts-regret-pf-2}
  \centering
  \begin{aligned}
         \mathcal{O} (\zeta_t^{1/2}) = \mathcal{O} (\sqrt{d\log(t)})(\sqrt{d\log(t)})=\mathcal{O}(d\log(t)).
    \end{aligned}
  \end{equation} 
     From Lemma~\ref{lemma:ts-regret-1}, the regret bound is of 
    \begin{equation} \label{eqn:ts-regret-pf-3}
  \centering
  \begin{aligned}
         \mathcal{O}(R_T)=&\mathcal{O} (\zeta_T^{1/2}\sum_{t-1}^T\sigma_{t-1}(\xbm_t)+\zeta_T\sqrt{T}) 
              = \mathcal{O}(d\log(T)\sqrt{T\gamma_T}+d^2\log^2(T)\sqrt{T})\\
               =&\mathcal{O}(d\log(T)\sqrt{T} (\sqrt{\gamma_T}+d\log(T)))
               =\mathcal{O}(d\log(T)\sqrt{T} \max\{\sqrt{\gamma_T},d\log(T)\}).
    \end{aligned}
  \end{equation} 
  Using $\gamma_T=\mathcal{O}\log^{d+1}(T)$ for SE kernel and $\gamma_T=\mathcal{O}(T^{\frac{d}{2\nu+d}}\log^{\frac{2\nu}{2\nu+d}}(T))$ for Matérn kernel from~\cite{vakili2021information} completes the proof.
\end{proof}
\begin{remark}
     Similar to Theorem~\ref{thm:compact-regretbound-rkhs} and Remark~\ref{remark:betterreg}, we obtain an improved regret bound compared to~\cite{chowdhury2017kernelized} as long as $d\log(T)=o(\gamma_T)$, which is the case for both SE and Matérn kernels. 
  Compared to GP-UCB, the regret bound for GP-TS is of $\log^{\frac{1}{2}}(T)$ worse, similar to the observation in~\cite{chowdhury2017kernelized} for finite $d$. 
  Further, the amount of improvement is similar to that of GP-UCB (Remark~\ref{remark:betterreg}).
\end{remark}

\section{Conclusions}\label{se:conclusion}
In this paper, we proved a pointwise bound on the prediction error of GP under the frequentist setting and Gaussian noise. We further proved improved regret bounds for GP-UCB and GP-TS compared to the state-of-the-art results.

\section*{Acknowledgments}
Prepared	by	LLNL	under	Contract	DE-AC52-07NA27344.
This	manuscript	has	been	authored	by	Lawrence	Livermore	National	Security,	LLC	under	Contract	No.	DE-AC52-07NA2	
7344	with	the	US.	Department	of	Energy.	The	United	States	Government	retains,	and	the	publisher,	by	accepting	the	
article	for	publication,	acknowledges	that	the	United	States	Government	retains	a	non-exclusive,	paid-up,	irrevocable,	
world-wide	license	to	publish	or	reproduce	the	published	form	of	this	manuscript,	or	allow	others	to	do	so,	for	United	
States	Government	purposes.

\medskip
\clearpage
\bibliography{bibliography}

\begin{thebibliography}{10}

\bibitem{lizotte2008}
D.~J. Lizotte.
\newblock {\em Practical bayesian optimization}.
\newblock PhD thesis, University of Alberta, Edmonton, Alberta, Canada, 2008.

\bibitem{jones2001taxonomy}
D.~R. Jones.
\newblock A taxonomy of global optimization methods based on response surfaces.
\newblock {\em Journal of global optimization}, 21:345--383, 2001.

\bibitem{bosurvey2023}
X.~Wang, Y.~Jin, S.~Schmitt, and M.~Olhofer.
\newblock Recent advances in {Bayesian} optimization.
\newblock {\em ACM Comput. Surv.}, 55(13s), jul 2023.

\bibitem{frazier2018}
P.~I. Frazier.
\newblock Bayesian optimization.
\newblock In {\em Recent advances in optimization and modeling of contemporary
  problems}, pages 255--278, October 2018.

\bibitem{wu2019hyperparameter}
Jia Wu, Xiu-Yun Chen, Hao Zhang, Li-Dong Xiong, Hang Lei, and Si-Hao Deng.
\newblock Hyperparameter optimization for machine learning models based on
  bayesian optimization.
\newblock {\em Journal of Electronic Science and Technology}, 17(1):26--40,
  2019.

\bibitem{mathern2021}
A.~Mathern, O.~S. Steinholtz, A.~Sj{\"o}berg, et~al.
\newblock {Multi-objective constrained Bayesian optimization for structural
  design}.
\newblock {\em Structural and Multidisciplinary Optimization}, 63:689--701,
  February 2021.

\bibitem{calandra2016}
R.~Calandra, A.~Seyfarth, J.~Peters, and M.~P. Deisenroth.
\newblock Bayesian optimization for learning gaits under uncertainty.
\newblock {\em Annals of Mathematics and Artificial Intelligence}, 76:5--23,
  February 2016.

\bibitem{srinivas2009gaussian}
Niranjan Srinivas, Andreas Krause, Sham~M Kakade, and Matthias Seeger.
\newblock Gaussian process optimization in the bandit setting: No regret and
  experimental design.
\newblock {\em arXiv preprint arXiv:0912.3995}, 2009.

\bibitem{chowdhury2017kernelized}
Sayak~Ray Chowdhury and Aditya Gopalan.
\newblock On kernelized multi-armed bandits.
\newblock In {\em International Conference on Machine Learning}, pages
  844--853. PMLR, 2017.

\bibitem{vakili2021scalable}
Sattar Vakili, Henry Moss, Artem Artemev, Vincent Dutordoir, and Victor
  Picheny.
\newblock Scalable thompson sampling using sparse gaussian process models.
\newblock {\em Advances in neural information processing systems},
  34:5631--5643, 2021.

\bibitem{vakili2021information}
Sattar Vakili, Kia Khezeli, and Victor Picheny.
\newblock On information gain and regret bounds in gaussian process bandits.
\newblock In {\em International Conference on Artificial Intelligence and
  Statistics}, pages 82--90. PMLR, 2021.

\bibitem{kandasamy2018parallelised}
Kirthevasan Kandasamy, Akshay Krishnamurthy, Jeff Schneider, and Barnab{\'a}s
  P{\'o}czos.
\newblock Parallelised bayesian optimisation via thompson sampling.
\newblock In {\em International conference on artificial intelligence and
  statistics}, pages 133--142. PMLR, 2018.

\bibitem{valko2013finite}
Michal Valko, Nathaniel Korda, R{\'e}mi Munos, Ilias Flaounas, and Nelo
  Cristianini.
\newblock Finite-time analysis of kernelised contextual bandits.
\newblock {\em arXiv preprint arXiv:1309.6869}, 2013.

\bibitem{janz2020bandit}
David Janz, David Burt, and Javier Gonz{\'a}lez.
\newblock Bandit optimisation of functions in the mat{\'e}rn kernel rkhs.
\newblock In {\em International Conference on Artificial Intelligence and
  Statistics}, pages 2486--2495. PMLR, 2020.

\bibitem{mutny2018efficient}
Mojmir Mutny and Andreas Krause.
\newblock Efficient high dimensional bayesian optimization with additivity and
  quadrature fourier features.
\newblock {\em Advances in Neural Information Processing Systems}, 31, 2018.

\bibitem{calandriello2019gaussian}
Daniele Calandriello, Luigi Carratino, Alessandro Lazaric, Michal Valko, and
  Lorenzo Rosasco.
\newblock Gaussian process optimization with adaptive sketching: Scalable and
  no regret.
\newblock In {\em Conference on Learning Theory}, pages 533--557. PMLR, 2019.

\bibitem{scarlett2017lower}
Jonathan Scarlett, Ilija Bogunovic, and Volkan Cevher.
\newblock Lower bounds on regret for noisy gaussian process bandit
  optimization.
\newblock In {\em Conference on Learning Theory}, pages 1723--1742. PMLR, 2017.

\bibitem{bull2011convergence}
Adam~D Bull.
\newblock Convergence rates of efficient global optimization algorithms.
\newblock {\em Journal of Machine Learning Research}, 12(10), 2011.

\bibitem{agrawal2013thompson}
Shipra Agrawal and Navin Goyal.
\newblock Thompson sampling for contextual bandits with linear payoffs.
\newblock In {\em International conference on machine learning}, pages
  127--135. PMLR, 2013.

\bibitem{wang2023regret}
Wenjia Wang, Xiaowei Zhang, and Lu~Zou.
\newblock Regret optimality of gp-ucb.
\newblock {\em arXiv preprint arXiv:2312.01386}, 2023.

\bibitem{williams2006gaussian}
Christopher~KI Williams and Carl~Edward Rasmussen.
\newblock {\em Gaussian processes for machine learning}, volume~2.
\newblock MIT press Cambridge, MA, 2006.

\bibitem{demmel1997applied}
James~W Demmel.
\newblock {\em Applied numerical linear algebra}.
\newblock SIAM, 1997.

\bibitem{lederer2019uniform}
Armin Lederer, Jonas Umlauft, and Sandra Hirche.
\newblock Uniform error bounds for gaussian process regression with application
  to safe control.
\newblock {\em Advances in Neural Information Processing Systems}, 32, 2019.

\end{thebibliography}
\bibliographystyle{unsrt}

\newpage

\begin{appendix}
\section{Proof for auxiliary results}\label{appdx:cplus}

\begin{lemma}\label{lem:lipschitz-cplus}
    If a function $f:\Rbb^n\to\Rbb$ is Lipschitz continuous on $O\subset\Rbb^n$ with Lipschitz constant $L$ in 1-norm, then its maximum with $0$ is also Lipschitz, \textit{i.e.}
    \begin{equation} \label{eqn:lip-plus1}
  \centering
  \begin{aligned}
     |f^+(\xbm) - f^+(\xbm')| \leq L \norm{\xbm-\xbm'}_1,
  \end{aligned}
\end{equation}
where $f^+(\xbm)=\max\{f(\xbm),0\}$ and  $\xbm,\xbm'\in O$.
\end{lemma} 
\begin{proof}
  From the definition of Lipschitz functions, we know
     \begin{equation} \label{eqn:lip-plus-pf-1}
  \centering
  \begin{aligned}
     |f(\xbm) - f(\xbm')| \leq L \norm{\xbm-\xbm'}_1,
  \end{aligned}
\end{equation}
   If $f(\xbm)>0$ and $f(\xbm')>0$, then 
     \begin{equation} \label{eqn:lip-plus-pf-2}
  \centering
  \begin{aligned}
     |f^+(\xbm) - f^+(\xbm')| = |f(\xbm) - f(\xbm')| \leq L \norm{\xbm-\xbm'}_1.
  \end{aligned}
\end{equation}
   If $f(\xbm)>0$ and $f(\xbm')\leq 0 $, then  
     \begin{equation} \label{eqn:lip-plus-pf-3}
  \centering
  \begin{aligned}
     |f^+(\xbm) - f^+(\xbm')| = f(\xbm) \leq f(\xbm) - f(\xbm')\leq  |f(\xbm) - f(\xbm')| \leq L \norm{\xbm-\xbm'}_1.
  \end{aligned}
\end{equation}
   Similarly, if $f(\xbm)\leq 0$ and $f(\xbm')> 0 $, then  
     \begin{equation} \label{eqn:lip-plus-pf-4}
  \centering
  \begin{aligned}
     |f^+(\xbm) - f^+(\xbm')| = f(\xbm') \leq f(\xbm') - f(\xbm)\leq  |f(\xbm) - f(\xbm')| \leq L \norm{\xbm-\xbm'}_1.
  \end{aligned}
\end{equation}
   Finally if $f(\xbm) \leq 0$ and $f(\xbm') \leq 0 $, then  
      \begin{equation} \label{eqn:lip-plus-pf-5}
  \centering
  \begin{aligned}
     |f^+(\xbm) - f^+(\xbm')| = 0 \leq  |f(\xbm) - f(\xbm')| \leq L \norm{\xbm-\xbm'}_1.
  \end{aligned}
\end{equation}
\end{proof}

\end{appendix}

\end{document}